\documentclass[a4paper,UKenglish,cleveref, autoref, thm-restate]{lipics-v2021}
\usepackage{makecell}
\usepackage{multirow}
\usepackage[ruled,vlined]{algorithm2e}

\usepackage{amsthm}
\usepackage{amsmath}
\usepackage{aligned-overset}
\usepackage{graphicx}
\usepackage{emptypage}
\usepackage{xcolor}

\usepackage{mathtools}

\usepackage{ bbold }

\crefname{ALG@line}{line}{lines}
\Crefname{ALG@line}{Line}{Lines}

\usepackage{tcolorbox}
\usepackage{tikz}
\usetikzlibrary{arrows.meta,shapes.geometric,positioning}

\renewcommand{\inf}{\mathop{\mathrm{inf}\vphantom{\mathrm{sup}}}}
\newcommand{\NN}{\ensuremath{\mathbb{N}}\xspace}

\newcommand{\EE}{\ensuremath{\mathbb{E}}\xspace}
\newcommand{\PP}{\ensuremath{\mathbb{P}}\xspace}
\newcommand{\GG}{\ensuremath{\mathbb{G}}\xspace}
\newcommand{\MM}{\ensuremath{\mathbb{M}}\xspace}

\newcommand{\A}{\mathcal{A}\xspace}

\newcommand{\T}{\mathcal{T}\xspace}

\renewcommand{\S}{\mathcal{S}\xspace}

\newcommand{\printValueFunction}[1]{\mathsf{#1}}
\newcommand{\Val}{\printValueFunction{Val}}

\newcommand{\Reach}{\printValueFunction{Reach}}
\newcommand{\safereach}{\printValueFunction{NotUntil}}

\newcommand{\Prob}{\delta\xspace}
\newcommand{\E}{\mathbb{E}\xspace}

\newcommand{\defas}{\coloneqq}

\newcommand{\maxp}{\mathrm{Max}}
\newcommand{\minp}{\mathrm{Min}}
\newcommand{\eps}{\varepsilon}

\newcommand{\BR}{\textnormal{BR}}

\newcommand{\ALG}{\textnormal{LeTuReGa}}

\title{PAC Learning in Turn-Based Stochastic Games with Reachability Objectives: A Decentralized Private Approach via Expected Conditional Distance} % Title of the assignment

\titlerunning{PAC Learning in Turn-Based Stochastic Games with Reachability Objectives} %TODO optional, please use if title is longer than one line

\author{Ali Asadi}{Institute of Science and Technology Austria (ISTA), Austria}{ali.asadi@ista.ac.at}{https://orcid.org/0009-0005-2839-953X}{}

\author{Krishnendu Chatterjee}{Institute of Science and Technology Austria (ISTA), Austria}{krishnendu.chatterjee@ista.ac.at}{https://orcid.org/0000-0002-4561-241X}{}

\author{Pavol Kebis}{Institute of Science and Technology Austria (ISTA), Austria}{pavol.kebis@ista.ac.at}{https://orcid.org/0000-0003-0561-1364}{}

\authorrunning{A. Asadi, K. Chatterjee, P. Kebis} %TODO mandatory. First: Use abbreviated first/middle names. Second (only in severe cases): Use first author plus 'et al.'

\Copyright{CC-BY} %TODO mandatory, please use full first names. LIPIcs license is "CC-BY";  http://creativecommons.org/licenses/by/3.0/


 %TODO mandatory; please add comma-separated list of keywords

\category{} %optional, e.g. invited paper

\relatedversion{} %optional, e.g. full version hosted on arXiv, HAL, or other respository/website
\funding{The research was partially supported 
by Austrian Science Fund (FWF) 10.55776/COE12, ERC CoG 863818 (ForM-SMArt), FWF-2022-SFB F8502 (SPyCoDe), and ERC-2020-AdG 101020093 (VAMOS) grants.} %optional, to capture a funding statement, which applies to all authors. Please enter author specific funding statements as fifth argument of the \author macro.

\nolinenumbers %uncomment to disable line numbering

\EventEditors{Ana Sokolova and Patrick Totzke}
\EventNoEds{2}
\EventLongTitle{37th International Conference on Concurrency Theory (CONCUR 2026)}
\EventShortTitle{CONCUR 2026}
\EventAcronym{CONCUR}
\EventYear{2026}
\EventDate{September 1--4, 2026}
\EventLocation{Liverpool, UK}
\EventLogo{}
\SeriesVolume{391}
\ArticleNo{3}

\begin{document}

\maketitle

%TODO mandatory: add short abstract of the document
\begin{abstract}
Reachability is the most fundamental logical objective, yet it is notoriously difficult to learn in reinforcement learning settings: even for Markov decision processes, PAC learning of reachability is impossible without additional assumptions. This difficulty also holds in turn-based stochastic games (TBSGs), where two adversarial players interact on a finite state space. In this work, we consider turn-based stochastic games with reachability objectives. For such settings, adversarial learning, in which players are adversarial even in the learning phase, is impossible. Therefore, the goal is to consider learning, in which both players learn the unknown model together. In this spirit, previous literature on PAC learning in TBSGs considers (a)~public information shared by both players; and (b)~centralized learning, which means that players share the same learning algorithm. In this work, our contribution is two-fold. First, we relax these strong assumptions and ensure learning: (i)~with private information not shared with the other player; and (ii)~decentralized learning where the players do not share the same learning algorithm. To the best of our knowledge, this work is the first positive result for decentralized and private information learning of TBSGs with reachability objectives. Second, we introduce a game-theoretic generalization of the Expected Conditional Distance (ECD) parameter, which measures the expected length of reaching the target set. We establish a polynomial-sample complexity bound with respect to the number of states, actions, ECD parameter, and inverses of error tolerance and failure probability.
\end{abstract}

\section{Introduction}
\textbf{Turn-Based Stochastic Games.} Turn-based stochastic games (TBSGs)~\cite{DBLP:journals/iandc/Condon92} are zero-sum turn-based games played over a finite state space by two adversarial players, $\maxp$ and $\minp$, along with randomness in the transition function. The state space is partitioned into two disjoint sets for $\maxp$ and $\minp$. At each time step, the player owning the current state chooses an action. The subsequent state is then determined by a probabilistic transition function. 
This model generalizes several classical formalisms such as Markov decision processes (MDP)~\cite{DBLP:books/wi/Puterman94}, which have only one player and stochastic uncertainty, and graph games~\cite{DBLP:journals/jacm/ChandraKS81,DBLP:conf/stoc/GurevichH82}, where the transition function collapses to Dirac distributions. \\
\textbf{Objectives.} In TBSGs, the interaction of players is guided by an \emph{objective function}, which formally captures the desired behaviour of the model. Objectives are typically categorized into: (a)~logical objectives, e.g., reachability, safety, and parity; and (b)~quantitative objectives, e.g., finite-horizon, discounted sum, and mean payoff. This work focuses on reachability objectives, which are the most fundamental logical objectives, i.e., given a set of target states, the objective requires that some target state is eventually visited.  It is important to distinguish reachability from discounted sum or finite-horizon objectives. Discounted sum objectives introduce a discount factor $\lambda < 1$, which effectively imposes a "soft" horizon. Finite-horizon objectives strictly bound the interaction to $L$ steps. In contrast, reachability is an unbounded property; a target might be reached after an arbitrarily large number of steps.\\
\textbf{Strategies and Values.} Strategies are recipes that define the choice of actions of the players. They are functions that, given a game history, return a distribution over actions. Given a TBSG and an objective, the value of player $\maxp$ at a state is the maximal expectation that the player can guarantee for the objective against all strategies of player $\minp$. A strategy is $\eps$-optimal if it guarantees the value up to additive error $\eps$.\\
\textbf{PAC Learning.} While classical model checking assumes a known model, in the reinforcement learning setting the model is unknown. The players must learn near-optimal strategies solely through interaction with a simulator. In this setting, the gold standard is Probably Approximately Correct (PAC) guarantees for learning near-optimal strategies~\cite{DBLP:journals/cacm/Valiant84}. The PAC-RL problem is defined as follows.
\begin{tcolorbox}
Can we design learning algorithms for both players such that for any error tolerance $\epsilon > 0$ and failure probability $p \in (0,1)$, the algorithms output strategies that are $\epsilon$-optimal with probability at least $1-p$? 
\end{tcolorbox}
Crucially, for the problem to be considered tractable, the number of samples required by the algorithm (called \emph{sample complexity}) must be polynomial in number of states and actions, inverse error tolerance $1/\epsilon$ and inverse failure probability $1/p$.\\
% \textcolor{red}{
% \subparagraph{General Impossibility.} \sout{For general TBSGs (and even MDPs), the learning landscape is constrained by severe hardness results. Specifically, PAC-RL for reachability objectives is known to be impossible without further assumptions~\cite{DBLP:conf/birthday/AlurBBJ22,DBLP:journals/corr/abs-2111-12679}. The difficulty arises from the sensitivity of the reachability value to arbitrarily small transition probabilities. A transition with probability $q \approx 0$ might be the unique path to the target, and distinguishing this $q$ from $0$ requires $\Omega(1/q)$ samples. Since $q$ can be infinitesimally small, the sample complexity cannot be bounded by any function in game parameters, the error tolerance, and failure probability. This stands in sharp contrast to discounted objectives, where the effective horizon is bounded by $1/(1-\lambda)$, or finite-horizon objectives, where the depth is fixed at $L$.
% }}
\textbf{Expected Conditional Distance.} Even in MDPs, the PAC-RL problem for reachability objectives is impossible in general~\cite{DBLP:conf/birthday/AlurBBJ22,DBLP:journals/corr/abs-2111-12679}. Thus, to circumvent this impossibility the literature considers further assumptions including prior knowledge on (a)~the topology of the underlying graph~\cite{DBLP:conf/rss/FuT14}; (b)~the minimum non-zero probability~\cite{DBLP:conf/cav/AshokKW19}; and (c) a parameter called the Expected Conditional Distance (ECD), which was introduced in~\cite{DBLP:conf/icml/SvobodaBC24} for MDPs. The ECD parameter provides a measure of the expected number of steps to reach the target. We generalize ECD to the TBSG setting. This generalization is quite subtle, as several natural generalizations of ECD to games fail to achieve PAC guarantees.  
%\textcolor{red}{\sout{For each strategy $\pi_{\minp}$ for the player $\minp$, ECD considers the strategies of player $\maxp$ that are best-responses for the reachability objective, i.e., they maximize the probability of reaching the target set, and among those ECD takes the one with the shortest expected reaching time conditioned on eventually reaching $T$. Then ECD takes the supremum over $\pi_{\minp}$, meaning the player $\minp$ can pick a strategy that makes even the shortest reachability-optimal $\maxp$ response take as long as possible (in conditional expectation) to reach $T$.}} 
Intuitively, if a game has a small ECD, it implies that if the target is reachable, it is reachable relatively quickly on average. This assumption excludes pathological games where the only optimal strategies involve waiting for exponentially many steps. Bounding the ECD allows us to truncate the infinite-horizon, converting the intractable reachability problem into a tractable finite-horizon approximation.\\
\textbf{Tractable Private and Decentralized Learning.} Adversarial learning, in which players are adversarial even in the learning phase, is impossible for TBSGs with reachability objectives: consider a game with an initial player-$\minp$ state and an additional action that immediately leads to the target. In the learning phase, $\minp$ chooses the trivial target-reaching action, which is never part of the optimal strategy, rendering learning useless. Since adversarial learning with PAC guarantees is impossible, the goal is to consider learning where both players learn the unknown model together. In this setting, \cite{DBLP:conf/cav/AshokKW19} established an anytime algorithm with the prior knowledge on (a)~the minimum non-zero transition probability; or (b)~the topology of the underlying graph. However, this prior work \cite{DBLP:conf/cav/AshokKW19} has two important limitations: First, it assumes (i)~public
information shared by both players; and (ii)~centralized learning where players share the same learning algorithm. Second, while the algorithm is anytime, it does not provide sample-complexity bounds for the PAC-RL problem.\\
\textbf{Motivation.} 
The motivation of this work is two-fold. The main motivation is to relax the above two strong assumption and ensure learning: (i)~with private information not shared with the other player; and (ii)~decentralized learning where the players do not share the same learning algorithm. Second, even in previous setting of centralized learning with public information, sample complexity bound was not established. The goal is to establish a polynomial-sample complexity bound with respect to the number of states, actions, ECD parameter, and inverses of error tolerance and failure probability.\\
\textbf{Our Contributions.} We address the above gaps by considering the decentralized private information setting of TBSGs with reachability objectives. We present a pair of algorithms for both players which are PAC-RL learnable with prior knowledge on the ECD parameter. The sample complexity of these algorithms is polynomial in the game parameters, the inverses of the error tolerance and failure probability, and the ECD parameter. To the best of our knowledge, this pair of algorithms is the first positive result for decentralized private information PAC-RL learning  in TBSGs with reachability objectives.\\
\textbf{Technical Contributions.} Our technical contributions are as follows. 
\begin{itemize}
    \item We generalize the Expected Conditional Distance (ECD) parameter to the TBSG setting. %We define $ECD_{\GG}^{\epsilon}$ as the best-achievable expected time to reach the target set under all $\epsilon$-optimal strategies for player $\maxp$, conditioned on the target actually being reached.
    \item We provide a reduction showing that, if the ECD of a game is small, we can approximate the reachability value using a finite-horizon reachability objective.
    \item We define a finite-horizon expanded game over state-step pairs that unfolds the horizon into the state space, enabling backward induction and local learning at each state-step.
    \item We present a learning procedure where both players use backward induction to learn local $\epsilon$-optimal actions one step at a time. The algorithm uses a best-arm identification routine at each state-step to identify $\epsilon$-optimal actions with high confidence.
    \item The algorithm iteratively constructs a set of strategies in stages. Each newly constructed strategy is added to the set used in subsequent stages to ensure that previously discovered state-steps of the game remain reachable while the players explore new state-steps.
    \item To explore new state-steps, we maintain a set of unexplored ones. These are treated as auxiliary target sets, incentivising the players to visit more of the state-step space.
    %\item Despite the decoupled and private nature of the learning environment where players do not share states or actions, we demonstrate that the players can remain synchronized throughout the stages. \todo Pavol. are these contributions?\todo
\end{itemize}
Proofs omitted due to space restrictions are provided in the Appendix.\\
\textbf{Technical Novelty.} The technical novelty of this work is two-fold. The first novelty is the appropriate definition of ECD for games. Second, the previous works rely on estimating the underlying probabilistic transitions, which is infeasible in private decentralized learning. Our approach carefully combines different techniques: best arm identification; tracking strategies for exploration; and backward induction to directly compute near-optimal strategies.\\
\textbf{Related Works.}
The intersection of formal verification and reinforcement learning has recently received significant attention. We summarize some related works as follows.
\begin{itemize}
    \item \emph{MDPs.} PAC guarantees for complex logical objectives, such as Linear Temporal Logic (LTL)~\cite{DBLP:conf/focs/Pnueli77}, has seen significant development but remains constrained by specific environmental assumptions, due to the inherent impossibility of learnability in general settings without additional assumptions~\cite{DBLP:conf/birthday/AlurBBJ22,DBLP:journals/corr/abs-2111-12679}. Early PAC learning algorithm presented in~\cite{DBLP:conf/rss/FuT14} required complete knowledge of the environment's topology. \cite{DBLP:conf/cav/AshokKW19} improved upon this by requiring only a lower bound on the minimum non-zero transition probability. More recently, \cite{DBLP:conf/aaai/PerezS024} has established PAC results using the mixing time of the environment. \cite{DBLP:conf/icml/SvobodaBC24} has introduced the ECD parameter and established PAC results relying on this parameter. 
    \item \emph{TBSGs.} Early works on learning in TBSGs focused mainly on quantitative objectives and did not provide PAC guarantees~\cite{DBLP:journals/mor/LakshmivarahanN81,DBLP:conf/icml/Littman94,DBLP:conf/ijcai/BrafmanT99}. For logical objectives, \cite{DBLP:conf/ijcai/WenT16} has established PAC learning algorithms for TBSGs with LTL objectives by combining the special case of almost-sure satisfaction of a specification with optimizing quantitative objectives. \cite{DBLP:conf/cav/AshokKW19} obtained PAC guarantees for reachability objectives by computing under- and over-approximation of values, originally introduced in~\cite{DBLP:conf/cav/KelmendiKKW18}. It is noteworthy that all these works consider the centralized public information setting.
\end{itemize}

\section{Preliminaries}
In this section, we define the notations of turn-based stochastic games and PAC learning.\\
\textbf{Notation.} For a positive integer $n$, the set $\{0, 1, 2, \ldots, n\}$ is denoted by $[n]$. Open and closed intervals of reals are denoted $(x, y) = \{a \in \mathbb{R}~|~ x < a < y\}$ and $[x, y] = \{a \in \mathbb{R} ~|~ x \leq a \leq y\}$, respectively. Sets are denoted by calligraphic letters, e.g., $\S, \A$. Elements of sets are denoted by lowercase letters, e.g., $s, a$.
%Random elements with values in these sets are denoted by uppercase letters, e.g., $S, A$ \todo this does not hold.
The set of probability distributions over a set $\S$ is denoted by $\Delta(\S)$. The set of natural numbers is $\mathbb{N} = \{0, 1, 2, \ldots\}$.

\subsection{Turn-Based Stochastic Games}
\begin{definition}[Turn-Based Stochastic Games]
   A \emph{turn-based stochastic game} (TBSG for short) is a tuple $\GG = (\S, \A, \Prob, \mu)$
    where
    \begin{itemize} 
        \item $\S = \S_\maxp \uplus \S_\minp$ is a finite set of states, partitioned into the set of player-$\maxp$ states $\S_\maxp$ and the set of player-$\minp$ states $\S_\minp$;
        \item $\A$ is a finite set of actions; 
        \item $\Prob \colon \S \times \A \to \Delta(\S)$ is a probabilistic transition function which, given a state and an action, assigns a probability distribution over the successor state; and
        \item $\mu \in \Delta(S)$ is a probability distribution over the initial state.
    \end{itemize}
\end{definition}
\textbf{Dynamic.} At the beginning, an initial state $s_0 \sim \mu$ is drawn, and the game proceeds as follows. In each step $\ell \in \mathbb{N}$, the owner of the state $s_\ell$ selects an action $a_\ell \in \A$, possibly at random, and the successor state $s_{\ell+1} \sim \Prob(s_\ell, a_\ell)$ is drawn.\\
\textbf{Histories and Plays.} A history is a finite sequence $h = (s_0, a_0, s_1, a_1, \ldots, s_L)$ of states and actions such that for all $\ell \in [L-1]$, we have $\Prob(s_\ell, a_\ell)(s_{\ell+1}) > 0$. A play is an infinite sequence of states and actions $\omega = (s_0, a_0, s_1, a_1, \ldots)$ such that, for all $\ell \in \mathbb{N}$, we have $\Prob(s_\ell, a_\ell)(s_{\ell+1}) > 0$. The set of all plays is denoted by $\Omega$.\\
\textbf{Strategies.} A strategy determines how a player chooses an action based on the history up to a given step. Formally, a strategy for a player $i \in \{\maxp, \minp\}$ is a function $\pi_i \colon (\S \times \A)^\star \times \S_i \to \Delta(\A)$. The set of all strategies is denoted by $\Pi_i$. A strategy is Markovian if it depends on the current state and current step of the play, i.e., $\pi_i \colon S_i \times \NN \to \Delta(\A)$. A strategy is pure if it prescribes deterministic actions, i.e., it corresponds to a function $\pi_i \colon (\S \times \A)^\star \times \S_i \to \A$. A strategy is memoryless if it decides only based on the current state, i.e., $\pi_i \colon \S_i \to \Delta(\A)$. A strategy is positional if it is pure and memoryless. Note that in TBSGs with reachability objectives positional strategies are as powerful as general strategies~\cite{DBLP:journals/iandc/Condon92}.  Given strategies for both players $\pi_\maxp$ and $\pi_\minp$, we denote the strategy profile by $(\pi_\maxp, \pi_\minp)$, and if the context is clear, we simply use $\pi$.\\
\textbf{Probability Measures.} For a history $h$, its cone is the set of plays where $h$ is their prefix. Given a strategy profile $\pi$ and an initial belief $\mu$, the unique probability measure over Borel sets of infinite plays is denoted by $\PP^\pi_{\mu}( \cdot )$, which is defined by Carathéodory's extension theorem by extending the natural definition over cones of plays~\cite{billingsley2012ProbabilityMeasurea}.\\
\textbf{Reachability Objectives.} An \emph{objective} in a TBSG is a Borel set of plays 
$\Phi \subseteq \Omega$ in the Cantor topology on~$\Omega$~\cite{kechrisClassicalDescriptiveSet1995}. 
In this work, we consider reachability objectives which lie in the first level of the Borel hierarchy.
Given a set of target states $\T$, the reachability objective requires that a target state 
is eventually visited, i.e., $\Reach(\T) \defas \{ \omega \in \Omega \colon \exists \ell \in \mathbb{N} \quad s_\ell \in \T \}$. 
The goal of player $\maxp$ is to maximize the probability of satisfying the objective, while the goal of player $\minp$ is to minimize it.

We now recall a fundamental determinacy for TBSGs with reachability objectives. 
\begin{theorem}[Determinacy~\protect{\cite{DBLP:journals/iandc/Condon92}}]
\label{thm:determinacy-in-tbsg}
For all TBSGs with a target set $\T \subseteq \S$, we have
\begin{align*}
    \sup_{\pi_\maxp \in \Pi_\maxp} \inf_{\pi_\minp \in \Pi_\minp} \PP_\mu^{\pi} (\omega \in \Reach(\T)) = \inf_{\pi_\minp \in \Pi_\minp} \sup_{\pi_\maxp \in \Pi_\maxp} \PP_\mu^{\pi} (\omega \in \Reach(\T)) \,.
\end{align*}
\end{theorem}
\textbf{Values.} \Cref{thm:determinacy-in-tbsg} implies that switching the quantifiers 
does not make a difference and leads to a unique notion of value. Formally, given a target set $\T$,
the value is a function of initial distribution
\begin{align*}
\Val_{R(\T)}^{\GG}(\mu) \defas \sup_{\pi_\maxp \in \Pi_\maxp} \inf_{\pi_\minp \in \Pi_\minp} \PP_\mu^{\pi} (\omega \in \Reach(\T)).
\end{align*}
We omit writing $\GG$ when clear from the context.\\
\textbf{Approximately Optimal Strategies.}
Given $\eps \ge 0$, a strategy $\pi_\maxp$ for player $\maxp$ is $\eps$-optimal 
if it guarantees the value up to an additive error $\eps$, i.e., 
if $\inf_{\pi_\minp \in \Pi_\minp} \PP_\mu^{\pi}(\omega \in \Reach(\T)) \ge \Val_{R(\T)}(\mu) - \eps$.
We denote the set of $\eps$-optimal strategies by $\Pi_{\maxp}^{\eps}$. 
In particular, we call a $0$-optimal strategy simply optimal. The definition of $\eps$-optimal strategies for player~$\minp$ is analogous.\\
\textbf{Best-responses.} For a player-$\minp$ strategy $\pi_\minp$, we define the set of best-responses for player $\maxp$ as 
\begin{align*}
    \BR(\pi_\minp) \defas \Big\{ \pi_\maxp \in \Pi_\maxp \colon &\PP_\mu^{(\pi_\maxp, \pi_\minp)} \left ( \omega \in \Reach(\T) \right ) =\\
    &\sup_{\pi'_\maxp \in \Pi_\maxp } \PP_\mu^{(\pi'_\maxp, \pi_\minp)} \left ( \omega \in \Reach(\T) \right ) \Big \}\,.
\end{align*}
The set of best-responses for player $\minp$ is defined analogously.

\subsection{Reinforcement Learning for TBSGs}
In the reinforcement learning setting for TBSGs with reachability objectives, we consider a scenario 
where the players have no information about the transition probabilities $\Prob$ or the initial 
distribution~$\mu$; only the parameters $\S$ and $\A$ are known to both players, and the players access the TBSG only through a simulator $\MM$. 
The goal of both players is to use 
\emph{learning algorithms} to find a near-optimal strategy profile.
In this work, the learning algorithms are \emph{decoupled}, i.e.,
each player has its own learning algorithm that does not communicate with the other player's algorithm.
%\textcolor{red}{\sout{We consider a \emph{cooperative} setting, which means that we can choose the algorithms for both players and design them so that they identify near-optimal strategies in a coordinated manner.}}
The assumption of \emph{private states} is another difference that distinguishes our setting from previously considered settings on TBSGs. 
We assume that the current state of a play is announced only to its owner and not to the other player. 
In the rest,
we formalize the notion of simulators, learning algorithms, and PAC-RL in this setting.

\begin{definition}[Simulators]\label{def:simulator}
Given a TBSG $\GG = (\S, \A, \Prob, \mu)$ with a target set $\T \subseteq \S$, a simulator $\MM$ stores the current state of the game, receives inputs from both players, performs actions, and outputs to players the new state to which the play is proceeded. Precisely, it works as follows:
\begin{enumerate}
    \item Any player $i$ can propose to terminate the simulator with a strategy $\pi_i$ by calling the procedure $\MM.propose(\pi_i)$. The simulator terminates with a strategy profile $\pi$ only if both players propose.
    \item $\MM$ informs both players that a new play has started;
    \item $\MM$ samples the initial state $s \sim \mu$;
    \item $\MM$ repeats the following:
    \begin{enumerate}
        \item the active player $i$ is $\maxp$ if $s \in \S_\maxp$ or $\minp$ if $s \in \S_\minp$;
        \item if $s \in \T$, both players are informed that the target was reached and the simulator returns to step 1;
        \item both players are informed who the active player is but \textbf{only the active player has access to the current state $s$};
        \item active player $i$ either (I)~chooses an action $a \in \A$ by calling the procedure $\MM.step(a)$; or (II) resets the game by calling the procedure $\MM.reset()$. Note that \textbf{$a$ is announced only to the simulator and not the other player}. The play proceeds with transitioning to a new state $s' \sim \Prob(s, a)$ and returning to step 4.
    \end{enumerate}
\end{enumerate}
\end{definition}

\begin{definition}[Learning Algorithms]
    A learning algorithm $\mathfrak{A}_i$ for a player $i \in \{\maxp, \minp\}$ is an algorithm 
    that interacts with the simulator $\MM$ by calling the procedures $\MM.step(a), \MM.reset(),$ and $\MM.propose(\pi_i)$ where $a \in \A,$ and $\pi_i$ is a player-$i$ strategy.
    %A learning algorithm induces a random sequence of strategies $\{\pi_i^t\}_{t \in \NN}$, where $\pi_i^t$ is the output strategy after $t$ iterations.
    A learning algorithm is 
    \emph{decoupled} if it does not communicate with the other player's learning algorithm.
\end{definition}

% \begin{remark}
%     If $\MM$ informs both players about the current state $s$, then the whole process of learning can be 
%     replaced by a single algorithm with full information.
% \end{remark}

\begin{definition}[PAC-RL]
    A pair of learning algorithms $(\mathfrak{A}_\maxp, \mathfrak{A}_\minp)$ is PAC-RL for reachability objectives if there
    exists a function $f$ such that
    for all $\eps, p \in (0, 1)$ and all TBSGs $\GG = (\S, \A, \Prob, \mu)$ with a target set $\T$, taking $N = f(|\S|, |\A|, \frac{1}{p}, \frac{1}{\eps})$, with probability at least $1 - p$, the simulator terminates with a strategy profile $(\pi_\maxp, \pi_\minp)$ after at most $N$ procedure calls where both strategies are $\eps$-optimal.
\end{definition}
\textbf{Sample Complexity.} The function $f$ in the definition of PAC-RL is called 
the sample complexity of the learning algorithms. If $f$ is a polynomial function, 
then we say that the learning algorithms have polynomial sample complexity.\\
%\textcolor{blue}{It is standard to take inverses of $p$ and $\varepsilon$ in order to avoid exponential dependence on the bit representation.}
\textbf{General Hardness.}
It is known that, even for MDPs with reachability objectives, there is no algorithm that is PAC-RL in general~\cite{DBLP:conf/birthday/AlurBBJ22,DBLP:journals/corr/abs-2111-12679}, meaning that there is no function $f$ that satisfies the condition of PAC-RL.
%\textcolor{red}{\sout{The impossibility result is due to the fact that in MDPs with reachability objectives, small perturbations in transition probabilities can cause large changes in the value function.}} 
In order to circumvent this hardness, we consider a parameter called \emph{Expected Conditional Distance} (ECD).

\section{Expected Conditional Distance}
\label{sec:ecmd}
In this section, we introduce a parameter for TBSGs called the \emph{Expected Conditional Distance} (ECD).
%We later use this parameter to reduce the PAC-RL problem for reachability objectives to the PAC-RL problem for finite-horizon reachability objectives.
The ECD parameter was previously studied for MDPs with reachability objectives~\cite{DBLP:conf/icml/SvobodaBC24}.
The main goal of this parameter is to reduce the PAC-RL for reachability to PAC-RL for finite-horizon reachability.
The generalization of this parameter to TBSGs is quite subtle, since we show below that several natural generalizations do not yield a suitable bound on the horizon.
%We present these generalizations and explain why they fail.
We then provide an appropriate generalization to TBSGs and give a reduction from PAC-RL with the ECD parameter to PAC-RL for finite-horizon reachability objectives.
Finally, we discuss several key aspects of our parameter which justify the usefulness: (a) its important properties that make it useful for PAC learning; (b) how it can be bounded using other classical parameters from the literature; and (c) how it compares with the well-studied stochastic shortest path parameter.
%We present several natural alternative generalizations of ECD to games that fail to provide an appropriate bound on the horizon of the game.

\begin{definition}[Alternative Generalizations]
\label{def:alt-gen}
Given a TBSG $\GG = (\S, \A, \Prob, \mu)$ with a target set $\T \subseteq \S$, consider the following alternative definitions of ECD:
\begin{align*}
        &\text{ECD}_\GG^1 \defas \sup_{\pi_\minp \in \Pi_\minp}\inf_{\pi_{\maxp} \in \Pi_\maxp }\text{ETR}_\GG(\pi_\maxp, \pi_\minp)\,,\\
        &\text{ECD}_\GG^2 \defas \sup_{\pi_\minp \in \Pi^0_\minp}\inf_{\pi_{\maxp} \in \BR(\pi_\minp)}\text{ETR}_\GG(\pi_\maxp, \pi_\minp)\,,\\
        &\text{ECD}_\GG^3 \defas \inf_{\pi_{\maxp} \in \Pi^0_\maxp}\sup_{\pi_\minp \in \Pi_\minp}\text{ETR}_\GG(\pi_\maxp, \pi_\minp)\,,\\
        &\text{ECD}_\GG^4 \defas \sup_{\pi_\minp \in \Pi_\minp}\inf_{\pi_{\maxp} \in \Pi^0_\maxp}\text{ETR}_\GG(\pi_\maxp, \pi_\minp)\,,\\
        &\text{ECD}_\GG^5 \defas \inf_{\pi_\minp \in \Pi_\minp}\sup_{\pi_{\maxp} \in \Pi_\maxp}\text{ETR}_\GG(\pi_\maxp, \pi_\minp)\,,
\end{align*}
where ETR$(\pi$) is the expected time to reach the target set using the strategy profile $(\pi_\maxp, \pi_\minp)$:
\[
    \text{ETR}_\GG(\pi_\maxp, \pi_\minp) \defas \E_\mu^\pi \left ( \arg\inf_{n \in \mathbb{N}}\mathbb{1}(s_n \in \T)~|~ (s_n)_{n \in \mathbb{N}} \cap \T \neq \emptyset \right ) \,.
\]
\end{definition}

These definitions fail in the following example for a reduction of PAC-RL for reachability to PAC-RL for finite-horizon reachability.

\begin{example}
     Consider a game, shown in \Cref{fig:example-to-fail}, with four states $\S = \{s_0, s_1, \top, \bot\}$ where $\top$ and $\bot$ are absorbing states. State $s_0$ belongs to $\maxp$ and $s_1$ belongs to $\minp$. The action set is $\A = \{a, b\}$. The target set is $\T = \{ \top \}$. In state $s_0$, playing action $a$ leads to $\top$ with probability $0.9$ and leads to $\bot$ with probability $0.1$, and playing action $b$ leads to $s_1$ with probability $1$. In state $s_1$, playing action $a$ leads to $\top$ with probability $0.8$ and leads to $\bot$ with probability $0.2$, and playing action $b$ leads to $\top$ with probability $0.001$ and self loops with probability $0.999$. The initial state is $s_0$. Therefore, in the case of the infinite-horizon version of the game, the optimal strategies for both players are to play the action $a$. However, for any finite-horizon game with horizon $L \le 100$, the optimal strategy for $\minp$ is to play action $b$. We need the ECD parameter to bound a horizon length for which a near-optimal strategy in the finite-horizon game is also near-optimal in the infinite-horizon game. However, all of the definitions above fail as their values are at most $2$. The values $\text{ECD}_\GG^1, \text{ECD}_\GG^2, \text{ECD}_\GG^3$, and $\text{ECD}_\GG^4$ are equal to $1$ for this game since the infimum over player-$\maxp$ actions selects action $a$ which ends the game immediately. The value of $\text{ECD}_\GG^5$ is $2$ since the infimum over player-$\minp$ actions selects the action $a$. It is noteworthy that changing the probabilities of action $b$ in the state $s_1$ makes the gap between the needed horizon and ECD values larger. In contrast, our definition of ECD provides a suitable bound on the horizon since $\text{ECD}_\GG = 1001$. 
\end{example}

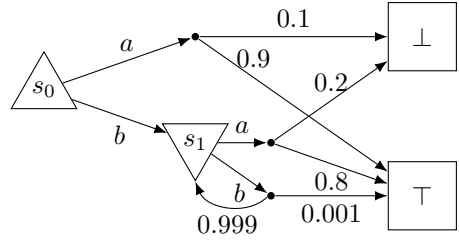
\begin{figure}
    \centering
    \begin{tikzpicture}[
            >=Latex,
            node distance=10mm and 12mm,
            player/.style={
                draw,
                regular polygon,
                regular polygon sides=3,
                minimum size=10mm,
                inner sep=0pt
            },
            playerdown/.style={
                player,
                shape border rotate=180
            },
            chance/.style={
                circle,
                fill=black,
                inner sep=1pt
            },
            terminal/.style={
                draw,
                minimum width=9mm,
                minimum height=9mm,
                inner sep=0pt
            }
        ]
        
        % Nodes
        \node[player]     (s0) at (0,0) {$s_0$};
        \node[chance]     (u)  at (2,0.7) {};
        \node[playerdown] (s1)  at (2,-0.7) {$s_1$};
        \node[chance]     (w)  at (3,-0.7) {};
        \node[chance]     (x)  at (3,-1.4) {};
        \node[terminal]   (T)  at (5,-1.4) {$\top$};
        \node[terminal]   (L)  at (5,0.7) {$\bot$};
        
        % Edges from left player node
        \draw[->] (s0) -- node[above] {$a$} (u);
        \draw[->] (s0) -- node[below] {$b$} (s1);
        
        % Upper stochastic node
        \draw[->] (u) -- node[above] {$0.1$} (L);
        \draw[->] (u) -- node[pos=0.15,right] {$0.9$} (T);
        
        % Lower player node
        \draw[->] (s1) -- node[above] {$a$} (w);
        \draw[->] (s1) -- node[below] {$b$} (x);
        
        % Middle stochastic node
        \draw[->] (w) -- node[above] {$0.2$} (L);
        \draw[->] (w) -- node[below] {$0.8$} (T);
        
        % Bottom stochastic node
        \draw[->] (x) to[bend left=55] node[below] {$0.999$} (s1);
        \draw[->] (x) -- node[below] {$0.001$} (T);
    
    \end{tikzpicture}
    \caption{A game where alternative generalizations of ECD fail}
    \label{fig:example-to-fail}
\end{figure}

\begin{definition}[Expected Conditional Distance]
\label{def:ecd}
    Given a TBSG $\GG = (\S, \A, \Prob, \mu)$ with a target set $\T \subseteq \S$, the \emph{expected conditional distance} is defined as follows.
    
    \[
        \text{ECD}_\GG \defas \sup_{\pi_\minp \in \Pi_\minp}\inf_{\pi_{\maxp} \in \BR(\pi_\minp)}\text{ETR}_\GG(\pi_\maxp, \pi_\minp)\,,
    \]
    where $\text{ETR}_\GG(\pi_\maxp, \pi_\minp)$ is defined as in \Cref{def:alt-gen}.
\end{definition}
\textbf{Description of ECD.} If the ECD parameter is bounded by $L$, then for all strategies for player~$\minp$, there exists a best-response strategy of player $\maxp$ that can reach the target set $\T$ in expected time at most $L$. More formally, the ECD parameter is defined as follows. 
% \textcolor{blue}{\sout{The player $\minp$ first chooses an arbitrary strategy $\pi_\minp$. Given $\pi_\minp$, ECD restricts the player $\maxp$ to strategies that are best responses for the reachability objective.} 
For any player-$\minp$ strategy, take the set of player-$\maxp$ strategies that are best-responses for the reachability objective.
Among these reachability-optimal strategies, ECD takes the one that minimizes the expected number of steps to reach the target, conditioned on the target being reached. Finally, ECD takes the maximum of this quantity over all player-$\minp$ strategies.

We now define the PAC learning framework with respect to the ECD parameter.

\begin{definition}[PAC-RL with ECD]
    A pair of learning algorithms $(\mathfrak{A}_\maxp, \mathfrak{A}_\minp)$ is PAC-RL with ECD if there
    exists a function $f$ such that
    for all $\eps, p \in (0, 1)$, all $L \in \NN$, and all TBSGs $\GG = (\S, \A, \Prob, \mu)$ 
    with a target set $\T \subseteq \S$ such that $\text{ECD}_\GG \le L$, 
    taking $\\N = f(|\S|, |\A|, \frac{1}{p}, \frac{1}{\eps}, L)$, 
    with probability at least $1 - p$, the simulator terminates with a strategy profile $(\pi_\maxp, \pi_\minp)$ after at most $N$ procedure calls where both strategies are $\eps$-optimal.
\end{definition}

\begin{remark}
    The difference between this definition and the standard PAC-RL definition 
    is the inclusion of the ECD parameter $L$ in the function $f$.
    % The key insight is that when the ECD is small, the learning algorithms
    % do not require to consider histories of length much longer than $L$ to find 
    % $\eps$-optimal strategies, which allows to circumvent the impossibility result for the general case.
\end{remark}

We now define an objective called finite-horizon reachability and show that the problem of PAC-RL with ECD for reachability objectives can be reduced to the problem of PAC-RL for finite-horizon reachability objectives.

\begin{definition}[Finite-horizon Reachability Objectives]
    Given a set of target states $\T$ and a time horizon $L \in \mathbb{N}$, the finite-horizon reachability objective requires that a target state is visited within the first $L$ steps, i.e., $\Reach_L(\T) \defas \{ \omega \in \Omega \colon \exists \ell \in [L] \quad s_\ell \in \T \}.$ We also admit $L \in \mathbb{R}$ in which case the target has to be visited within the first $\lfloor L \rfloor$
    where $\lfloor x \rfloor$ is the biggest number $n \in \mathbb{N}$ such that $n \leq x$.
    We denote
    $\Val_{R_L(\T)}^{\GG}(\mu) \defas \sup_{\pi_\maxp \in \Pi_\maxp} \inf_{\pi_\minp \in \Pi_\minp} \PP_\mu^{\pi} (\omega \in \Reach_L(\T))$. We omit writing $\GG$ when clear from the context.
\end{definition}

We similarly define the PAC-RL for finite-horizon reachability objectives.

\begin{definition}[PAC-RL for Finite-horizon Reachability Objectives]
    A pair of learning algorithms $(\mathfrak{A}_\maxp, \mathfrak{A}_\minp)$ is PAC-RL for finite-horizon reachability objectives if there
    exists a function $f$ such that
    for all $\eps, p \in (0, 1)$, all TBSGs $\GG = (\S, \A, \Prob, \mu)$ with a target set $\T$, and all time-horizons $L \in \NN$, taking $N = f(|\S|, |\A|, \frac{1}{p}, \frac{1}{\eps}, L)$, with probability at least $1 - p$, the simulator terminates with a strategy profile $(\pi_\maxp, \pi_\minp)$ after at most $N$ procedure calls where both strategies are $\eps$-optimal.
\end{definition}

\begin{proposition}
\label{prop:inifinte-to-finite-approx}
    Let $\pi$ be a strategy profile such that $\text{ETR}_\GG(\pi) \le L$. Then, for all $\eps > 0$ we have
    \[
        \left |\PP_\mu^\pi  ( \omega \in \Reach_{\frac{L}{\eps}}(\T)  ) - \PP_\mu^\pi \left (\omega \in \Reach(\T) \right ) \right | \le \eps\,.
    \]
\end{proposition}
\begin{proof}
    Recall that $\text{ETR}_\GG(\pi) = \E_\mu^\pi \left (\arg\inf_{n \in \mathbb{N}}\mathbb{1}(s_n \in \T)~|~ (s_n)_{n \in \mathbb{N}} \cap \T \neq \emptyset \right )$. By Markov's inequality and $\text{ETR}_\GG(\pi) \le L$, we have $\PP_\mu^\pi \left (\arg\inf_{n \in \mathbb{N}}\mathbb{1}(s_n \in \T) > \frac{L}{\eps} ~|~ (s_n)_{n \in \mathbb{N}} \cap \T \neq \emptyset \right ) \le \eps$, which yields the result.
\end{proof}

\begin{restatable}{theorem}{reductiontofinite}
\label{the:reduction}
    If a pair of learning algorithms $(\mathfrak{A}_\maxp, \mathfrak{A}_\minp)$ is PAC-RL for finite-horizon reachability objectives, then it is PAC-RL with ECD for reachability objectives.
\end{restatable}

\begin{proof}[Proof Sketch.]
    Let $H=2(L+1)/\varepsilon$. Run the finite-horizon PAC-RL algorithm with horizon $H$ and error tolerance $\varepsilon/2$. With probability at least $1-p$, it returns a profile $\pi^\star$ that is $\varepsilon/2$-optimal for the finite-horizon reachability game.
    Since $\text{ECD}_\GG\le L$, for every player-$\minp$ strategy $\pi_\minp$ there exists a best-response $\pi_\maxp$ of $\maxp$ such that $\text{ETR}_\GG(\pi_\maxp, \pi_\minp) \le L+1$. By Proposition 13, truncating reachability to horizon $H=2(L+1)/\varepsilon$ changes the reachability probability of such a best response by at most $\varepsilon/2$. Therefore, $|\Val_{R(\T)}(\mu) - \Val_{R_{\frac{2(L + 1)}{\eps}}(\T)}(\mu)| \le \frac{\eps}{2}$.
    Combining this with the $\varepsilon/2$-optimality of $\pi^\star$ in the finite-horizon game gives that $\pi^\star$ is $\varepsilon$-optimal for the original reachability objective. Hence PAC-RL for finite-horizon reachability implies PAC-RL with ECD for reachability.
\end{proof}

We now discuss several key aspects of the ECD parameter which justify why we use this parameter in this work.\\
\textbf{Properties of ECD.} our ECD definition has two important properties: First, for every game it is finite. Second, it has a meaningful intuition to bound the horizon of the game, i.e., it captures that for every strategy of player $\minp$, there exists a counter-strategy of player $\maxp$ such that (i)~the counter-strategy is optimal for the reachability objectives with respect to the strategy of player $\minp$; and (ii)~the expected time to reach is small.\\
\textbf{Estimation of ECD.} 
    A related parameter is the minimum 
    non-zero transition probability $p_{\min}$ of a TBSG $\GG$. This parameter has been used in the context of PAC learning for TBSGs with reachability objectives~\cite{DBLP:conf/cav/AshokKW19}.
    The minimum non-zero transition probability $p_{\min}$ provides a bound on the expected time to reach the target set $\T$. 
    Indeed, if $p_{\min} > 0$, then for any strategy profile $\pi$, we have $\text{ETR}_\GG(\pi) \le (1/p_{\min})^{|\S|}$. However, this worst-case bound is exponential, while the ECD parameter can be much smaller. Better bounds require more information about the game. Since we provide the theoretical foundation in this work, model-dependent estimation of this parameter is subject for future work.\\
\textbf{Stochastic Shortest Path.} 
    A closely-related parameter to ECD is the \emph{stochastic shortest path} (SSP) parameter~\cite{DBLP:journals/mor/BertsekasT91}.
    The difference between SSP and ECD is that, in SSP, player $\maxp$ requires to reach the target set $\T$ as soon as possible, while player $\minp$
    wants to delay the reachability of the target set $\T$ as much as possible. The SSP parameter measures non-reaching plays as having infinite cost. Therefore, this parameter can be infinite. In contrast, our definition guarantees that the parameter is always finite. Finiteness of the parameter is necessary for the reduction to finite-horizon games.

\section{PAC-RL for Finite-horizon Reachability}
\label{sec:pac-rl-for-finite}

\newcommand{\esign}{\varepsilon^{significant}}
\newcommand{\eerr}{\eps_{bai}}
\newcommand{\eest}{\eps_{emp}}
\newcommand{\eimp}{\varepsilon^{improvement}}
\newcommand{\unexplored}{U}
\newcommand{\wrong}{Wrong}
\newcommand{\wellexp}{\text{Estim}}
\newcommand{\bestarm}{\text{Best-Arm}}
\newcommand{\updatebai}{\text{Update-BAI}}
\newcommand{\initbai}{\text{Initialise-BAI}}
\newcommand{\counted}{C}

In this section, we present a pair of algorithms $(\ALG_\maxp$, $ \ALG_\minp)$ for PAC-RL of TBSGs with finite-horizon reachability. $\ALG$ stands for Learning Turn-based Reachability Games.

\begin{theorem}\label{the:main}
The pair of algorithms $(\ALG_\maxp$, $ \ALG_\minp)$ is PAC-RL for finite-horizon reachability objectives with sample complexity $O\left(\frac{|\S|^3L^7|\A|\log(|\S|^2L^2/p)}{\eps^3}\right)$.
\end{theorem}

\Cref{the:reduction,the:main} imply a result for reachability objectives with ECD assumption. 

\begin{corollary}\label{cor:main}
The pair of algorithms $(\ALG_\maxp$, $\ALG_\minp)$ is PAC-RL with ECD for reachability objectives with sample complexity $O\left(\frac{|\S|^3L^7|\A|\log({|\S|^2L^2}/(p\eps^2))}{\eps^{10}}\right)$.
\end{corollary}

\begin{proof}
To obtain $\varepsilon$-optimal strategies for the infinite-horizon reachability, we use finite-horizon algorithms with the length of the game $2(L+1)/\varepsilon$ (see the proof of \Cref{the:reduction}).
\end{proof}
\textbf{Significance.} \Cref{cor:main} establishes that reachability becomes PAC learnable in turn-based stochastic games in a decentralized and private information setting under bounded ECD. To the best of our knowledge, this is the first result that (i) handles decentralized and private learning; or (ii) provides explicit polynomial sample complexity bounds. 

This section is organized as follows. We first recall some algorithms from bandit learning literature.
%Then, we present the naive algorithm, and finally present the algorithm with polynomial sample complexity. 
Then, we describe the $\ALG$ algorithms and finally, we prove \Cref{the:main}.

%In this section, we present two algorithms for PAC-RL of TBSGs with finite-horizon reachability. The first algorithm is a naive approach which reduces the problem to two-player zero-sum bandit learning. However, this approach yields exponential sample complexity, which is not desirable. The second algorithm is more technically involved and yeilds polynomial sample complexity. This section is organized as follows. We first recall some algorithms from bandit learning literature. We then present the naive algorithm, and finally present the algorithm with polynomial sample complexity. 

% . It extends the state space to state-step pairs
% It iteratively runs best-arm-identification routine for all reachable state-step pairs in a backward manner.

\subsection{Best-Arm Identification}\label{sec:bandit}

In this subsection, we recall a problem in the bandit learning literature and an optimal solution for it.
\emph{The best arm identification bandit learning} problem asks to find an $\eps$-optimal arm with probability $1 - p$.
%The second problem is two-player zero-sum bandit learning. In this setting, there is an unknown matrix game in which both players try to find the $\eps$-optimal strategies with probability $1 - p$.
%For a single player who has a finite set of options and each option has a stochastic reward, the problem is usually referred to as the Best-Arm-Identification problem.
%\subparagraph{Best-Arm-Identification bandit learning.}
Let $\mathcal{A}$ be a set of arms where each arm is associated with an unknown value $r: \mathcal{A} \rightarrow [0, 1]$.
A player can sample an arm $a \in \mathcal{A}$ to obtain a random reward $R \in \{0, 1\}$ such that $\E(R) = r(a)$.
We say an algorithm can identify an $\eps$-optimal arm with sample complexity $f$ and confidence $1 - p$, if for any set of arms $\mathcal{A}$ and any $p, \eps \in [0, 1]$, after $f(|\mathcal{A}|, 1/p, 1/\varepsilon)$ samples, with probability at least $1 - p$ it outputs a candidate arm $a'$ such that $|\max_{a}r(a) - r(a')| \leq \eps$.
A basic approach is to sample each arm $\frac{\log(|\A|/p)}{\eps^2}$ times, estimate its unknown value, and then select the best arm. The guarantees follow directly from Hoeffding's bound. A better sample complexity is achieved by the Median Elimination algorithm \cite{DBLP:journals/jmlr/Even-DarMM06}[Theorem 10], presented in \Cref{alg:bai} in \Cref{app:sec:algorithms}.
See \Cref{lem:bai} for the formal statement. This sample complexity matches the lower bound for this problem \cite{DBLP:journals/jmlr/MannorT04}.

\begin{lemma}[\cite{DBLP:journals/jmlr/Even-DarMM06}[Theorem 10]]
\label{lem:bai}
For a set of arms $\A$, a value function $r: \A \rightarrow [0, 1]$, error and confidence $\eps, p \in (0, 1)$, \Cref{alg:bai} identifies an $\eps$-optimal arm with confidence $1 - p$ and sample complexity $O\left(\frac{|\A|\log(1/p)}{\eps^2}\right)$.
\end{lemma}

%For multiple players, one has to define how the outcome depends on the joint behaviour of all agents. In our case, we can focus on zero-sum objectives as learning the optimal strategy in two-player zero sum matrix games is an active area of research. Here, we consider the case of bandit-feedback, that is, after both players independently choose their arms, they only receive a Bernoulli reward of their joint action.

\subsection{Algorithm}

This section presents a pair of algorithms for PAC-RL of TBSGs with finite-horizon reachability objectives. First, we give an overview of the main techniques used in the algorithm. We then provide a more detailed description. The pseudocode of the algorithm is provided in \Cref{algorithm:main}. The correctness of the algorithm is proven in the next subsection.\\
\textbf{Algorithm Overview.} Firstly, the algorithm \emph{extends the set of states to state-steps}, i.e., it learns which action is good enough for every state-step pair, where step is bounded by the horizon.
The algorithm keeps track of \emph{unexplored state-steps}.
Initially, all state-steps are considered unexplored except for the target set. The set of unexplored state-steps shrinks over time, and it is \emph{treated as a target} to enhance exploration.
In the learning process, the algorithm learns how to visit more state-steps and identifies good-enough actions for each of them.
The procedure follows in stages.
In each stage, the algorithm constructs a new strategy by \emph{backward induction}.
For each state-step, it uses a \emph{best arm identification routine}, which proposes a local $\varepsilon$-optimal action with high confidence.
The algorithm always learns \emph{one step at a time}, fixing the strategy in the rest of the game.
At the end of the induction, a candidate strategy is constructed.
This strategy is then used to discover new state-steps.
If no new state-steps are discovered from the perspective of the player, this player proposes to the simulator to terminate the algorithm with the recently constructed strategy.
The procedure terminates only when both players propose to the simulator to terminate.

% \textcolor{red}{\sout{
% While the setting is uncoupled (the current state of the game is not publicly known), the algorithms are \emph{synchronised} in the sense that both players know the current stage and the current depth of the induction.}}

We now define some notions used in our algorithm.
\begin{definition}[Expanded Game]
For a given TBSG $\GG = (\S, \A, \Prob, \mu)$ with a target set $\T \subseteq \S$ and a time horizon $L \in \NN$, we define the expanded game $\GG' \defas (\S', \A, \Prob', \mu')$ where 
\begin{itemize}
    \item  $\S' \defas \S_\maxp' \uplus \S_\minp'$ where $\S_i'\defas \{(s, \ell) \colon s \in \S_i, \ell \in [L]\}$ for $i \in \{\maxp, \minp\}$;
    \item For all states $s, s' \in \S$, actions $a \in \A$ and steps $\ell, \ell' \in [L]$, the transition function $\Prob'$ is defined as
    \[
        \Prob'((s, \ell), a)(s', \ell') \defas \begin{cases}
            \Prob(s, a)(s') & \text{if } \ell \in [L-1] \land \ell' = \ell+1\\
            1 & \text{if } \ell = L \land \ell' = L \land s' = s\\
            0 & \text{otherwise;}
        \end{cases}
    \]
    % \
    % $\Prob: \S' \times \A \rightarrow \Delta(\S')$ is defined for every $s, t \in \S, a \in \A, \ell \in [L-1]$ as $ \Prob'((s, \ell), a)(t, \ell+1) \defas \Prob(s, a)(t)$ and $\Prob'((s, L), a)(t, L) \defas \mathbb{1}(s = t)$;
    \item For all states $s \in \S$ and steps $\ell \in [L]$, the initial distribution~$\mu'$ is defined as
    \[
        \mu'(s, \ell) \defas \begin{cases}
            \mu(s) & \text{if } \ell = 0\\
            0 & \text{otherwise.}
        \end{cases}
    \]
\end{itemize}
The target set is defined as $\T_{\maxp} \defas \T \times [L]$.
We also define a target set for the player $\minp$ as $\T_{\minp} \defas \{(s , L) ~|~ s \not\in \T\}$.
\end{definition}

The expanded game is the original game accompanied by a counter. A play starts with the counter value of 0, and in every step the counter is incremented. The counter is bounded by $L$, which means the state stays invariant after $L$ steps. Consequently, the expanded game is equivalent to the original game for the finite-horizon $L$. Moreover, the fact that the state is not changed after $L$ steps implies that the finite-horizon variant has the same value as the infinite-horizon for the expanded game. Thus, we obtain the following result.

\begin{proposition}
\label{pro:expanded}
For a given TBSG $\GG = (\S, \A, \Prob, \mu)$ with a target set $\T \subseteq \S$ and a time horizon $L \in \NN$, we have
$\Val_{R_L(\T)}^{\GG} = \Val_{R_L(\T_\maxp)}^{\GG'} = \Val_{R(\T_\maxp)}^{\GG'}$.
\end{proposition}

\begin{remark}
\label{rem:positional}
Recall that positional strategies are as powerful as general strategies for TBSGs with reachability objectives. Thus, by \Cref{pro:expanded}, we only consider positional strategies in the expanded game, and we need to consider Markovian strategies in the original game.
\end{remark}

\begin{definition}
Given an expanded TBSG game $\GG = (\S', \A, \Prob', \mu')$ and two disjoint sets $U, V \subseteq \S'$, we define $\safereach(U, V) \defas \{\omega \in (\S' \times \A)^* \times \S': \exists \ell \geq 0, s_\ell \in V, \forall j \in [\ell-1]: s_j \not\in U\}$. Intuitively, $\safereach(U, V)$ is the set of finite plays that avoid reaching any state from $U$ until a state from $V$ is reached.
%\textcolor{blue}{I changed 'SafeReach' to 'NotUntil'}
\end{definition}

In the algorithm, we use constants which we define below.

\begin{definition}[Constants]
\label{def:alg-constants}
Let $\GG = (\S, \A, \Prob, \mu)$ be a TBSG with finite-horizon $L$ and $p, \eps \in [0, 1]$ be the confidence and error of the PAC guarantees. We define the constants used in the algorithm as follows.
\begin{itemize}
    \item $\eest \defas \frac{\varepsilon}{8|\S|L}$;
    \item $\eerr \defas \frac{\varepsilon}{2L}$;
    %\item $K \defas  \frac{32C|\S|L^3 |\A|\log(|\S|^2L^2/p)}{\varepsilon^3}$
    \item $C$ is the constant given by the best arm identification algorithm constant that is implicitly present in the $O$-notation~\cite{DBLP:journals/jmlr/Even-DarMM06}[Theorem 10]; and
    \item $K \defas \frac{C |\A| \log(|\S|^2 L^2/p)}{\eest \eerr^2}$.
    %\item $\eimp = \frac{\varepsilon}{L}$
\end{itemize}
\end{definition}

We are now able to explain the algorithms in detail.\\
\textbf{Algorithm Details. }
Pseudocode of the algorithms is given in \Cref{algorithm:main}. We describe the algorithm for player $i \in \{\maxp, \minp\}$. The algorithm starts by initialising the set of unexplored state-steps $U_i^0$ to be any state owned by player $i$ which is not in the last step (Line~\ref{line:initun}) and the strategies $\pi_i^0, \ldots, \pi_i^{|\S'|}$ (Line~\ref{line:initsigma}). A strategy $\pi_i^q$ is constructed in the stage $q$ using backward induction (Line~\ref{line:changesigma}) and it is used in all of the following stages $q+1, q+2, \ldots$ for the purpose of exploration (Line~\ref{line:sigma}).
After the initialisation, the algorithm runs at most $|\S'|$ stages (Line~\ref{line:q}) and in each stage it performs backward induction on the length of the game $L$ (Line~\ref{line:length}). In a stage $q$, after it performs an induction, it checks whether the set of unexplored state-steps has shrinked or not (Line~\ref{line:conditiontermination}). If not, it means the strategy learnt in the stage $q-1$ did not explore anything new, which makes it a good candidate for an $\varepsilon$-optimal strategy. However, this holds only if the set of unexplored states is unchanged for both players in the same stage, which results in the termination of the procedure.
The backward induction is split into a sampling phase (Line~\ref{line:initbai} to Line~\ref{line:count}) and an analysis phase (Line~\ref{line:s} to Line~\ref{line:changesigma}). Let the stage be $q$ and the level of induction be $\ell$.
In the sampling phase, the algorithm learns the best action for a state-step $(s, \ell)$ where $s \in \S_i$.
To achieve this, the algorithm uses formerly constructed strategies $\pi_i^{0}, \ldots, \pi_i^{q-1}$ in the first $\ell - 1$ steps of the game.
In the step $\ell$ it plays according to a best-arm identification routine that tries various actions to determine the best one with high confidence.
In the steps $\ell +1 $ onwards, it plays according to the currently learnt strategy $\pi_i^q$ which is being inductively constructed (Line~\ref{line:strategy}).
The algorithm samples $qK$ plays, that is, $K$ plays for every formerly constructed strategy $\pi_i^{r}$ for $r \in [q-1]$ (Line~\ref{line:sample}). The algorithm tracks whether player $i$ is successful in a particular play (Line~\ref{line:result}). For player $\maxp$, this happens when the play reaches a target state or an unexplored state-step. For player $\minp$, it happens either when the target states are completely avoided, or when an unexplored state-step is reached before a target state is reached. For every state $s$, the algorithm keeps track of how many times the best arm identification routine was called for that state (Line~\ref{line:count}).
In the analysis phase of the induction, the algorithm changes the strategy $\pi_i^q$ to play according to the result of the best-arm identification routine for the step $\ell$ in a state $s$ (Line~\ref{line:changesigma}). This happens only if the state was visited a sufficient number of times (Line~\ref{line:condition2}).
Furthermore, if a state $s$ was visited even higher number of times (Line~\ref{line:condition1}), the state-step $(s, \ell)$ is removed from the set of unexplored state-steps (Line~\ref{line:remove}).\\
\textbf{Comparison with existing work.} 
In the algorithm design, we drew inspiration from \cite{DBLP:conf/colt/DaskalakisGZ23}. However, our work differs significantly from theirs.
First, we consider turn-based games with reachability objectives  where PAC-RL guarantees are impossible in general, while they consider concurrent discounted-sum games which are easy in PAC learning.
Second, they use an adversarial bandit learning routine, while we use a best-arm identification routine.
Third, our approach removes the need to estimate the visitation distribution, it is simpler in general, and mainly, the complexity of our algorithm is more efficient than theirs.

\begin{algorithm}
\SetAlgoLined
\LinesNumbered
\KwData{$\S_i, |\S|, \A,p,  \eps, L$} $\unexplored^0_i \gets \S_i \times [L-1]$ \tcp*[r]{set of unexplored states}\label{line:initun}

%$\Pi \gets [\sigma^{\text{uniform}}]$ ordered list      \tcp*[r]{initialize strategy cover}
$\pi^q_i \gets \pi^{\text{uniform}}$ for all $q \in 0, 1, \ldots, |\S'|$\label{line:initsigma} \tcp*[r]{memoryless uniform strategy}
$\counted_i^{q}(s, \ell) \gets 0$ for all $q \in [|\S'|], (s, \ell) \in \S'$\label{line:initcounter} \tcp*[r]{visit counter of $(s, \ell)$ in stage $q$}
%$\wellexp^{q, \ell} \gets \emptyset$ for all $q \in [|\S|L], \ell \in [L]$\;
\For{$q \in 1, \ldots, |\S'|$}{\label{line:q}  
    $\unexplored^{q}_i \gets \unexplored^{q-1}_i$\label{line:previousun}\;       
    \For {$\ell \in L-1, \ldots, 1$} {\label{line:length}
        Initialise a best-arm identification  (BAI) routine for all $s \in \S_i$ with $(\eerr, \frac{p}{|\S^2|L^2})$ parameters \label{line:initbai}\;
        %$\pi^{BAI}_i(s) \gets $BAI$_s.$get-arm$()$ for all $s \in \S_i$ \label{line:firstbai}\;                
        %$\counted^q_i(s, l, a) \gets 0$ for every $a \in A$\;
        %$Wins^q_i(s, l, a) \gets 0$ for every $a \in A$\;
        \For {$\pi_i \in \pi_i^0, \ldots, \pi_i^{q-1}$} {\label{line:sigma}
            \For {$k \in 1, \ldots, K$} {\label{line:k}
                Let $\pi_i^{BAI}(s)$ be the action that the BAI routine wants to sample at the state $s$, for all $s \in \S_i$\;
                Define a positional strategy for any $s \in \S_i, j \in [L]$ $\pi_i'(s, j) \gets
                \begin{cases} \pi_i(s, j) &\text{if the step } j < \ell \\
                \pi_i^{\text{BAI}}(s) &\text{if the step } j = \ell \\
                \pi_i^q(s, j) &\text{if the step } j > \ell
                \end{cases}$\label{line:strategy}\;
                Interact with the simulator $\MM$ to sample a play $(s_1, a_1, s_2, a_2, \ldots, s_{J-1}, a_{J-1}, s_J)$ using the strategy $\pi_i'$. Notice that $J < L$ only if $s_J \in \T$\label{line:sample}\; 
                \If{$J > \ell$ and $s_\ell \in \S_i$}{\label{line:condition0}
                    $Play \gets ((s_{\ell+1}, \ell+1), a_{\ell+1}, \ldots, (s_J, J))$\label{line:play}\;
                    $Winning \gets
                    \begin{cases}
                    \safereach(\emptyset, U_\maxp^{q} \cup \T_\maxp)&\text{ if } i = \maxp\\
                    \safereach(\T_\maxp, U_\minp^{q} \cup \T_\minp)&\text{ if } i = \minp
                    \end{cases}$\label{line:winning}\;
                    $Result \gets \mathbb{1}(Play \in Winning)$\label{line:result}\;
                    Update the BAI routine at the state $s_\ell$ with $(a_\ell, Result)$ \label{line:bai} \tcp*[r]{Note that $a_\ell$ was requested by $\pi_i^{BAI}(s_\ell)$}
                    %$\pi_i^{\text{BAI}}(s_{\ell}) \gets \updatebai(s_{\ell}, a_{\ell}, Result)$\label{line:bai}\;
                    Increment $\counted_i^{q}(s_{\ell}, \ell)$\label{line:count}\;
                }
            }
        }         
        \For {$s \in \S_i$} {\label{line:s}
            \If {$\counted_i^{q}(s, \ell) \geq 3K\eest$} {\label{line:condition1}
                $\unexplored_i^{q} \gets \unexplored_i^{q} / \{(s, \ell)\}$\label{line:remove}\;
            }
            \If {$\counted_i^{q}(s, \ell) \geq K\eest $} {\label{line:condition2}
                Make $\pi_i^q$ to play the action suggested by BAI at the state $(s, \ell)$\label{line:changesigma}\;
            }
        }
    }
    %Add $\sigma^{q}$ to $\Pi$\label{line:add}\;
    \If {$\unexplored_i^{q} = \unexplored_i^{q-1}$}  {\label{line:conditiontermination}
        Call $\MM.propose(\pi_i^{q-1})$\label{line:termination}\;
    }
}
%Return $\sigma^{SL}$.
\caption{Algorithm $\ALG_i$ for player $i \in \{\maxp, \minp\}$}
\label{algorithm:main}
\end{algorithm}

\subsection{Proof of \texorpdfstring{\Cref{the:main}}{Theorem~15}}
\textbf{Overview of the proof.} We proceed as follows:
(a) we prove that the learning simulation terminates (\Cref{lem:termination});
(b) we define some notations that we use in the next steps (\Cref{def:qvalue});
(c) we define an event (\Cref{def:events}) that occurs with high probability (\Cref{lem:event}); and
(d)~we show that under this event, the strategy profile proposed by the algorithms is $\varepsilon$-optimal for the finite-horizon reachability game (\Cref{lem:unexplored,lem:constructed,lem:real}).

\begin{lemma}
\label{lem:termination}
    The learning simulation of algorithms $(\ALG_\maxp, \ALG_\minp)$ terminates after at most $|\S'|$~stages. 
\end{lemma}
\begin{proof}
    At the end of every stage $q$, if $\unexplored^q_i = \unexplored^{q-1}_i$, then the algorithm $\ALG_i$ proposes $\pi_i^{q-1}$ as the candidate strategy. The learning simulation terminates (Line~\ref{line:termination}) if both algorithms propose a candidate strategy. Otherwise, at least one state-step has to be removed either from $\unexplored^q_\maxp$ or $\unexplored^q_\minp$. However, this cannot happen more than $|\S'|$ times as $|\unexplored^0_\maxp \cup \unexplored^0_\minp| \leq |\S'|$ which means the simulation terminates after at most $|\S'|$ stages. 
\end{proof}

\begin{definition}
    We define $f \in \{1, \ldots, |\S'|\}$ as the stage $q-1$, i.e., the second-to-last stage before termination.
\end{definition}
%In \Cref{def:qvalue}, we define some useful notions we use throughout the section. 
%\Cref{lem:connection} connects these notions with the values computed by the algorithm. 
%The main definition of the section is the event $E$ \Cref{def:events}, which is a collection of random events sufficient to prove \Cref{the:main}. In \Cref{lem:event}, we prove that the probability of event $E$ is $1 - p$. 
%Optimality of the strategies outputed by the algorithm is proved in \Cref{lem:real}. But first, we prove optimality of the outputed strategies with respected to a constructed game with different reachability objectives (\Cref{lem:constructed}). Though this constructed game is not zero-sum, the difference between the values of the constructed game and the real game is marginal as we prove in \Cref{lem:unexplored}.
%The whole proof consists of moving between the real game and a constructed games
\textbf{Notions of $V$ and $Q$ values.} We define the value $V_{B, W}^\pi(s, \ell)$ which is the probability of reaching a state-step from $W \subseteq \S'$ while avoiding $B \subseteq \S'$, starting in a state-step $(s, \ell)$ and playing according to the expanded game $\GG'$ and a strategy profile $\pi$. The value $Q_{B, W}^\pi(s, \ell, a)$ is defined similarly to $V_{B, W}^\pi(s, \ell)$ but the first action taken is $a$.

\begin{definition}\label{def:qvalue}
Let $\GG' = (\S', \A, \Prob', \mu')$ be an expanded game for an original game $\GG = (\S, \A, \Prob, \mu)$. Let $B, W \subseteq \S'$ be two disjoint subsets and $\pi$ a positional strategy profile for the expanded game.
For all states $s \in \S$ and steps $\ell \in [L]$, we inductively define
\begin{align*}
V_{B, W}^{\pi}(s, \ell) \defas 
    \begin{cases}
    1 &\text{ if } (s, \ell) \in W\\
    0 &\text{ if } (s, \ell) \in B \text{ or } (s, \ell) = (s, L) \not\in W \\
    Q_{B, W}^{\pi}(s, \ell, \pi(s, \ell)) &\text{ otherwise}
    \end{cases}
\end{align*}
and for all states $s \in \S$, steps $\ell \in [L-1]$, and actions $a \in \A$, we have
\begin{align*}
Q_{B, W}^{\pi}(s, \ell, a) \defas \sum_{s' \in \S} \Prob(s, a)(s') V_{B, W}^{\pi}(s', \ell+1)\,.
\end{align*}
We also denote $V_{B, W}^{\pi} \defas \sum_{s \in \S}\mu(s)V_{B, W}^{\pi}(s, 0)$.

\Cref{lem:connection} connects these notions with the values computed by the algorithms.
%We postpone the proof of the following Lemma to the appendix.

\begin{restatable}{proposition}{backwardinduction}
\label{lem:connection}
For all positional strategy profiles $\pi$ for the expanded game, all states $s \in \S$, and all sets $B, W \subseteq \S'$, the following statements hold.
\begin{itemize}
    \item For all steps $\ell \in [L]$, we have
    \begin{align*}
        &V_{B, W}^{\pi}(s, \ell)
        = \PP_{(s, \ell)}^{\pi} \left( \left (s_{\ell}, a_{\ell}, \ldots, s_{L} \right ) \in \safereach \left (B,W \right ) \right )\,;
    \end{align*}
    
    \item for all steps $\ell \in [L-1]$ and actions $a \in \A$, we have
    \begin{align*}
        &Q_{B, W}^{\pi}(s, \ell, a)
        = \PP_{\Prob'((s, \ell), a)}^{\pi} \left( \left (s_{\ell+1}, a_{\ell+1}, \ldots, s_{L} \right ) \in \safereach \left (B,W \right ) \right )\,.
    \end{align*}
\end{itemize}
\end{restatable}

\begin{proof}[Proof Sketch.]
The proof is by backward induction on $\ell$. At the last step , the claim is immediate from the definition. For $\ell < L$, by the induction hypothesis and the law of total probability, the quantity $Q_{B, W}^{\pi}(s, \ell, a) = \sum_{s' \in \S} \Prob(s, a)(s') V_{B, W}^{\pi}(s', \ell+1)$ is the probability of satisfying $\safereach(B,W)$ after taking action $a$. Then $V^\pi_{B,W}(s,\ell)$ follows directly because it is $0$ on $B$, $1$ on $W$, and otherwise equals $Q^\pi_{B,W}(s,\ell,\pi(s,\ell))$. Thus both recursive definitions coincide with the reach-avoid probabilities.
\end{proof}

\end{definition}
\textbf{Event $E$.} We define an event $E$ which is a collection of conditions. To do so, we first define some useful notations.

\begin{definition}
We define $\counted^{q}(s, \ell) \defas \sum_{i \in \{\maxp, \minp\}}\counted_i^{q}(s, \ell)$ and $\unexplored^q \defas \unexplored_\maxp^q \cup \unexplored_\minp^q$ for all $q \in [f], s \in \S, \ell \in [L].$ Intuitively, $\counted^q(s, \ell)$ is the number of times the pair $(s, \ell)$ is visited in the stage $q$. By $\unexplored^q$, we denote all unexplored state-step pairs.
\end{definition}

%Notice that after running the algorithm, for all $q \in [f]$, $(s, \ell) \in \S'$, at least one of the values $C_\maxp^q(s, \ell)$ and $C_\minp^q(s, \ell)$ is zero. This is due to the fact that $C_i^q(s, \ell)$ is incremented only if the state $s$ belongs to the player $i$.

\begin{definition}\label{def:events}
After both algorithms terminate, we say an event $E$ happens if for all stages $q \in [f]$ and every state $(s, \ell) \in \S'$, all of the following conditions hold:
\begin{enumerate}[(a)]
    \item 
    \begin{equation}\label{eq:probability}
           \left|\counted^{q}(s, \ell) -  K\sum_{r \in [q-1]} V_{\emptyset, \{(s, \ell)\}}^{\pi^r}
           \right| \leq K\eest;
           %\PP_{\mu}^{\sigma^r, \tau^r}(\omega \in \Reach(\{(s, \ell)\})
        \end{equation}
    \item if $\counted^{q}(s, \ell) \geq 3K\eest$ then $\counted^{q'}(s, \ell) \geq K\eest$ for all stages $q' \in  \{q, \ldots, f\}$; and
    \item for all $a \in \A$, if $s \in \S_\maxp$ then 
    \begin{align*}
    Q_{\emptyset, \unexplored_\maxp^{q} \cup \T_\maxp}^{\pi^q}(s, \ell, a) - Q_{\emptyset, \unexplored_\maxp^{q} \cup \T_\maxp}^{\pi^q}(s, \ell, \pi_\maxp^q(s, \ell))< \eerr
    \end{align*}
    and if $s \in \S_\minp$ then 
    \begin{align*}
    Q_{\T_\maxp, \unexplored_\minp^{q} \cup \T_\minp}^{\pi^q}(s, \ell, a) - Q_{\T_\maxp, \unexplored_\minp^{q} \cup \T_\minp}^{\pi^q}(s, \ell, \pi_\minp^q(s, \ell))< \eerr\,.
    \end{align*}
\end{enumerate}
\end{definition}

\begin{restatable}{lemma}{event}
\label{lem:event}
The probability of the event $E$ is at least $1-p$.
\end{restatable}

\begin{proof}[Proof Sketch.]
    The event $E$ consists of three parts. It is enough to show that (a) and (c) hold with probability at least $1-p/2$, while (b) follows deterministically from (a). A union bound then gives the result.
    For part (a), fix a stage $q$ and a state-step $(s,\ell)$. The count $C_q(s,\ell)$ is exactly the total number of visits to $(s,\ell)$ in the $K$ samples taken for each earlier profile $\pi^r$. Hence it can be written as a sum of independent indicators whose expectation is $K \sum_{r \in [q-1]} V_{\emptyset, \{(s, \ell)\}}^{\pi^r}$.
    Applying Hoeffding’s inequality shows that $C_q(s,\ell)$ concentrates around this expectation within $K\varepsilon_{\mathrm{emp}}$, and taking a union bound over all stages, steps, and states gives probability at least $1-p/2$.
    For part (c), once $(s,\ell)\notin U^q_i$, parts (a) and (b) ensure that the best-arm identification routine at $(s,\ell)$ has been called enough for its $(\varepsilon_{\mathrm{bai}}, p/(|S|^2L^2))$ guarantee. Therefore, by \Cref{lem:connection}, the selected action is $\varepsilon_{\mathrm{bai}}$-optimal with the stated confidence. Agian, a union bound gives probability at least $1-p/2$.
\end{proof}
\textbf{Near-optimality of $\pi^f$.} In the following, we condition on the event $E$. \Cref{lem:unexplored} bounds the probability of reaching the set of state-steps $\unexplored^f$ under the strategy profile $\pi^f$.
%Later, we use this lemma to say that treating $\unexplored^f$ as a target does not change the value of the game by much.
\Cref{lem:constructed} states that $\pi^f$ is $L \eest$-optimal when treating $\unexplored^f$ as a target.
\Cref{lem:real} combines the two lemmas to show that the strategy profile $\pi^f$ is $\eps$-optimal in the expanded game with finite-horizon reachability objectives.

\begin{lemma}\label{lem:unexplored}
Under the event $E$, we have $V_{\emptyset, \unexplored^f}^{\pi^{f}} \leq 4|\S'|\eest$.
\end{lemma}
\begin{proof}
First, the termination condition implies that $\unexplored^f = \unexplored^{f+1}$.
For any $(s, \ell) \in \S'$, we have that $(s, \ell) \in \unexplored^{f}$ iff $(s, \ell) \in \unexplored^{f+1}$ if and only if $(s, \ell)$ has not been removed from $\unexplored_\maxp^{f+1}$ nor $\unexplored_\minp^{f+1}$. Therefore, for all $(s,\ell) \in \unexplored^f$, we have $\counted^{q}(s, \ell) < 3K\eest$ for all stages $q \in \{1, \ldots, f+1\}$.
Using the property $(a)$ of the event $E$ for the stage $f+1$, we get that
$\sum_{r \in [f]} V_{\emptyset, \{(s, \ell)\}}^{\pi^r} \leq 4\eest$
which implies in particular that $V_{\emptyset, \{(s, \ell)\}}^{\pi^{f}} \leq 4 \eest.$
Using the fact that 
$V_{\emptyset,\unexplored^f}^{\pi^{f}} \leq \sum_{(s, l) \in \unexplored^f} V_{\emptyset, \{(s, \ell)\}}^{\pi^{f}}$
we obtain the desired inequality.
\end{proof}
\begin{restatable}{lemma}{expandedgamevalue}\label{lem:constructed}
Under the event $E$, the following statements hold.

\begin{itemize}
    \item For all strategies $\pi_\maxp$ for the player $\maxp$, we have
    \[
        V_{\emptyset, \unexplored_\maxp^f \cup \T_\maxp}^{\pi_\maxp, \pi_\minp^{f}} - V_{\emptyset, \unexplored_\maxp^f \cup \T_\maxp}^{\pi^{f}} \leq L \eerr\,;
    \]
    \item for all strategies $\pi_\minp$ of the player $\minp$, we have  
    \[
        V_{\T_\maxp, \unexplored_\minp^f \cup \T_\minp}^{\pi_\maxp^f, \pi_\minp} - V_{\T_\maxp, \unexplored_\minp^f \cup \T_\minp}^{\pi^{f}} \leq L \eerr\,.
    \]
\end{itemize}
\end{restatable}

\begin{proof}[Proof Sketch.]
Use backward induction on $\ell$ and prove the stronger bound\\
$V_{\emptyset, \unexplored_\maxp^f \cup \T_\maxp}^{\pi_\maxp, \pi_\minp^{f}}( s, \ell) - V_{\emptyset, \unexplored_\maxp^f \cup \T_\maxp}^{\pi^{f}}( s, \ell) \leq (L - l)\eerr$.
The base case is immediate because both values are  fixed. For the induction step, the continuation error is bounded by $(L-\ell-1)\eerr$ by the induction hypothesis. If $s\in S_{\max}$, there is at most one additional local error from the choice of $\pi^f_{\max}(s,\ell)$, and event $E$(c) gives that this action is $\eerr$-optimal. Hence the total loss is at most $(L-\ell)\varepsilon_{\mathrm{bai}}$. If $s\in S_{\min}$, no extra local loss is incurred, so the same bound follows. Evaluating this at the initial distribution gives the first inequality. The second inequality follows from symmetric arguments.
\end{proof}
\begin{restatable}{lemma}{expandedvalue}
\label{lem:real}
Under the event $E$, the following statements hold.
    \[
        \sup_{\pi_\maxp \in \Pi_\maxp} \PP_\mu^{\pi_\maxp, \pi^f_\minp} \left ( \omega \in \Reach_L(\T) \right) - \Val_{R_L}(\mu) \le \eps\,;
    \]
    \[
        \Val_{R_L}(\mu) - \inf_{\pi_\minp \in \Pi_\minp} \PP_\mu^{\pi^f_\maxp, \pi_\minp} \left ( \omega \in \Reach_L(\T) \right) \le \eps\,.
    \]
\end{restatable}

\begin{proof}[Proof Sketch.]
By \Cref{lem:connection}, finite-horizon reachability can be written using the $V$-values: $\PP_\mu^\pi \left ( \omega \in \Reach_L(\T) \right ) = V_{\emptyset, \T_\maxp}^{\pi} = 1 - V_{\T_\maxp, \T_\minp}^{\pi}$.
For $\maxp$, enlarge the target set from $\T_{\max}$ to $U^f_{\max}\cup \T_{\max}$. \Cref{lem:constructed} says that $\pi^f$ is $L\varepsilon_{\mathrm{bai}}$-optimal for this auxiliary target, and \Cref{lem:unexplored} says that the reachability probability to the unexplored set $U^f$ under $\pi^f$ is at most $4|S'|\varepsilon_{\mathrm{emp}}$. Hence, for any player-$\maxp$ strategy $\pi_\maxp$, we have $V_{\emptyset, \T_\maxp}^{\pi_\maxp, \pi_\minp^{f}} - V_{\emptyset, \T_\maxp}^{\pi^{f}} \leq L\eerr + 4|\S'|\eest \leq \eps$. The inequality for $\minp$ is proven analogously. Use the auxiliary target $U^f_{\min}\cup \T_{\min}$, apply \Cref{lem:constructed}, and then remove the unexplored-set error using \Cref{lem:unexplored}. Thus the strategies of both players are $\varepsilon$-optimal in the finite-horizon game.
\end{proof}
\begin{proof} [Proof of \Cref{the:main}]
We prove the correctness and sample complexity of the algorithms.\\
\emph{Correctness.} By \Cref{lem:termination}, the algorithms terminate and by \Cref{lem:real}, the algorithms output $\eps$-optimal strategies.
Thanks to \Cref{pro:expanded}, the value of the original game and the expanded game coincide.\\
\emph{Sample Complexity.} We bound the number of procedure calls to the simulator $\MM$. Every sampled play at Line~\ref{line:sample} corresponds to at most $L+1$ procedure calls (either $\MM.step(\cdot)$ or $\MM.reset()$). Furthermore, there are at most $|\S'|+1$ number of $\MM.propose(\cdot)$ calls.
We bound the number of sampled plays at Line~\ref{line:sample}.
The loop at Line~\ref{line:q} is iterated at most $|\S|L$ number of times, the loop at Line~\ref{line:length} at most $L$ times, the loop at Line~\ref{line:sigma} at most $|\S|L$ times, the loop at Line~\ref{line:k} at most $K$ times, which means the number of plays sampled is bounded by
\begin{align*}
|\S|L \cdot L \cdot |\S|L \cdot K &= |\S|^2L^3\frac{C|\A|\log(|\S|^2L^2/p)}{\eest \eerr^2}\\ 
&=32|\S|^3L^6\frac{C|\A|\log(|\S|^2L^2/p)}{\eps^3} \in O\left(\frac{|\S|^3L^6|\A|\log(|\S|^2L^2/p)}{\eps^3}\right)
\end{align*}
Multiplying this by $L+1$ and adding $|\S'|+1$ yields the result.
\end{proof}
\textbf{Concluding Remarks.}
In this work, we consider the PAC learning of turn-based stochastic games with reachability objectives. We provide algorithms that ensure learning: (a)~with private information; and (b)~decentralized setting. Moreover, we generalize the ECD parameter from MDPs to games and establish a polynomial-sample complexity bound with respect to the number of states, actions, ECD parameter, and inverses of error tolerance and failure probability. This framework suggests several interesting open problems: (i)~extending to concurrent stochastic games; and (ii)~the setting where samplings are drawn from an arbitrary state rather than relying on simulator which restarts from the initial distribution.

\bibliography{sources}

@inproceedings{DBLP:conf/cav/AshokKW19,
  author       = {Pranav Ashok and
                  Jan Kret{\'{\i}}nsk{\'{y}} and
                  Maximilian Weininger},
  editor       = {Isil Dillig and
                  Serdar Tasiran},
  title        = {{PAC} Statistical Model Checking for Markov Decision Processes and
                  Stochastic Games},
  booktitle    = {{CAV} 2019, New York City, NY, USA, July 15-18, 2019},
  series       = {Lecture Notes in Computer Science},
  volume       = {11561},
  pages        = {497--519},
  publisher    = {Springer},
  year         = {2019},
  url          = {https://doi.org/10.1007/978-3-030-25540-4\_29},
  doi          = {10.1007/978-3-030-25540-4\_29},
  bibsource    = {dblp computer science bibliography, https://dblp.org}
}

@inproceedings{DBLP:conf/icml/SvobodaBC24,
  author       = {Jakub Svoboda and
                  Suguman Bansal and
                  Krishnendu Chatterjee},
  title        = {Reinforcement Learning from Reachability Specifications: {PAC} Guarantees
                  with Expected Conditional Distance},
  booktitle    = {{ICML} 2024,
                  Vienna, Austria, July 21-27, 2024},
  publisher    = {OpenReview.net},
  year         = {2024},
  bibsource    = {dblp computer science bibliography, https://dblp.org}
}

@inproceedings{DBLP:conf/colt/DaskalakisGZ23,
  author       = {Constantinos Daskalakis and
                  Noah Golowich and
                  Kaiqing Zhang},
  editor       = {Gergely Neu and
                  Lorenzo Rosasco},
  title        = {The Complexity of Markov Equilibrium in Stochastic Games},
  booktitle    = {{COLT} 2023,
                  12-15 July 2023, Bangalore, India},
  series       = {Proceedings of Machine Learning Research},
  volume       = {195},
  pages        = {4180--4234},
  publisher    = {{PMLR}},
  year         = {2023},
  bibsource    = {dblp computer science bibliography, https://dblp.org}
}

@article{DBLP:journals/jmlr/Even-DarMM06,
  author       = {Eyal Even{-}Dar and
                  Shie Mannor and
                  Yishay Mansour},
  title        = {Action Elimination and Stopping Conditions for the Multi-Armed Bandit
                  and Reinforcement Learning Problems},
  journal      = {J. Mach. Learn. Res.},
  volume       = {7},
  pages        = {1079--1105},
  year         = {2006},
  bibsource    = {dblp computer science bibliography, https://dblp.org}
}

@book{billingsley2012ProbabilityMeasurea,
    title = {{Probability and Measure}},
    author = {Billingsley, Patrick},
    year = {2012},
    publisher = {Wiley},
    address = {Hoboken, NJ, USA},
}

@book{kechrisClassicalDescriptiveSet1995,
    title = {{Classical Descriptive Set Theory}},
    author = {Kechris, Alexander S.},
    year = {1995},
    publisher = {Springer},
    address = {New York, NY, USA},
    doi = {10.1007/978-1-4612-4190-4},
}

@article{DBLP:journals/iandc/Condon92,
  author       = {Anne Condon},
  title        = {The Complexity of Stochastic Games},
  journal      = {Inf. Comput.},
  volume       = {96},
  number       = {2},
  pages        = {203--224},
  year         = {1992}
}

@inproceedings{DBLP:conf/birthday/AlurBBJ22,
  author       = {Rajeev Alur and
                  Suguman Bansal and
                  Osbert Bastani and
                  Kishor Jothimurugan},
  title        = {A Framework for Transforming Specifications in Reinforcement Learning},
  booktitle    = {Principles of Systems Design},
  series       = {Lecture Notes in Computer Science},
  volume       = {13660},
  pages        = {604--624},
  publisher    = {Springer},
  year         = {2022}
}

@article{DBLP:journals/corr/abs-2111-12679,
  author       = {Cambridge Yang and
                  Michael L. Littman and
                  Michael Carbin},
  title        = {Reinforcement Learning for General {LTL} Objectives Is Intractable},
  journal      = {CoRR},
  volume       = {abs/2111.12679},
  year         = {2021}
}

@book{DBLP:books/wi/Puterman94,
  author       = {Martin L. Puterman},
  title        = {Markov Decision Processes: Discrete Stochastic Dynamic Programming},
  series       = {Wiley Series in Probability and Statistics},
  publisher    = {Wiley},
  year         = {1994}
}

@article{DBLP:journals/jacm/ChandraKS81,
  author       = {Ashok K. Chandra and
                  Dexter Kozen and
                  Larry J. Stockmeyer},
  title        = {Alternation},
  journal      = {J. {ACM}},
  volume       = {28},
  number       = {1},
  pages        = {114--133},
  year         = {1981}
}

@inproceedings{DBLP:conf/stoc/GurevichH82,
  author       = {Yuri Gurevich and
                  Leo Harrington},
  title        = {Trees, Automata, and Games},
  booktitle    = {{STOC}},
  pages        = {60--65},
  publisher    = {{ACM}},
  year         = {1982}
}

@inproceedings{DBLP:conf/focs/Pnueli77,
  author       = {Amir Pnueli},
  title        = {The Temporal Logic of Programs},
  booktitle    = {{FOCS}},
  pages        = {46--57},
  publisher    = {{IEEE} Computer Society},
  year         = {1977}
}

@inproceedings{DBLP:conf/rss/FuT14,
  author       = {Jie Fu and
                  Ufuk Topcu},
  title        = {Probably Approximately Correct {MDP} Learning and Control With Temporal
                  Logic Constraints},
  booktitle    = {Robotics: Science and Systems},
  year         = {2014}
}

@inproceedings{DBLP:conf/aaai/PerezS024,
  author       = {Mateo Perez and
                  Fabio Somenzi and
                  Ashutosh Trivedi},
  title        = {A {PAC} Learning Algorithm for {LTL} and Omega-Regular Objectives
                  in MDPs},
  booktitle    = {{AAAI}},
  pages        = {21510--21517},
  publisher    = {{AAAI} Press},
  year         = {2024}
}

@article{DBLP:journals/mor/LakshmivarahanN81,
  author       = {S. Lakshmivarahan and
                  Kumpati S. Narendra},
  title        = {Learning Algorithms for Two-Person Zero-Sum Stochastic Games with
                  Incomplete Information},
  journal      = {Math. Oper. Res.},
  volume       = {6},
  number       = {3},
  pages        = {379--386},
  year         = {1981}
}

@inproceedings{DBLP:conf/icml/Littman94,
  author       = {Michael L. Littman},
  title        = {Markov Games as a Framework for Multi-Agent Reinforcement Learning},
  booktitle    = {{ICML}},
  pages        = {157--163},
  publisher    = {Morgan Kaufmann},
  year         = {1994}
}

@inproceedings{DBLP:conf/ijcai/BrafmanT99,
  author       = {Ronen I. Brafman and
                  Moshe Tennenholtz},
  title        = {A Near-Optimal Poly-Time Algorithm for Learning a class of Stochastic
                  Games},
  booktitle    = {{IJCAI}},
  pages        = {734--739},
  publisher    = {Morgan Kaufmann},
  year         = {1999}
}

@inproceedings{DBLP:conf/ijcai/WenT16,
  author       = {Min Wen and
                  Ufuk Topcu},
  title        = {Probably Approximately Correct Learning in Stochastic Games with Temporal
                  Logic Specifications},
  booktitle    = {{IJCAI}},
  pages        = {3630--3636},
  publisher    = {{IJCAI/AAAI} Press},
  year         = {2016}
}

@inproceedings{DBLP:conf/cav/KelmendiKKW18,
  author       = {Edon Kelmendi and
                  Julia Kr{\"{a}}mer and
                  Jan Kret{\'{\i}}nsk{\'{y}} and
                  Maximilian Weininger},
  title        = {Value Iteration for Simple Stochastic Games: Stopping Criterion and
                  Learning Algorithm},
  booktitle    = {{CAV} {(1)}},
  series       = {Lecture Notes in Computer Science},
  volume       = {10981},
  pages        = {623--642},
  publisher    = {Springer},
  year         = {2018}
}

@article{DBLP:journals/jmlr/MannorT04,
  author       = {Shie Mannor and
                  John N. Tsitsiklis},
  title        = {The Sample Complexity of Exploration in the Multi-Armed Bandit Problem},
  journal      = {J. Mach. Learn. Res.},
  volume       = {5},
  pages        = {623--648},
  year         = {2004},
  bibsource    = {dblp computer science bibliography, https://dblp.org}
}

@article{DBLP:journals/cacm/Valiant84,
  author       = {Leslie G. Valiant},
  title        = {A Theory of the Learnable},
  journal      = {Commun. {ACM}},
  volume       = {27},
  number       = {11},
  pages        = {1134--1142},
  year         = {1984}
}

@article{DBLP:journals/mor/BertsekasT91,
  author       = {Dimitri P. Bertsekas and
                  John N. Tsitsiklis},
  title        = {An Analysis of Stochastic Shortest Path Problems},
  journal      = {Math. Oper. Res.},
  volume       = {16},
  number       = {3},
  pages        = {580--595},
  year         = {1991}
}

\newpage
\appendix

\section{Proofs of \texorpdfstring{\Cref{sec:ecmd}}{Section~\ref{sec:ecmd}}}
\reductiontofinite*
\begin{proof}
    Consider a TBSG $\GG = (\S, \A, \Prob, \mu)$, error tolerance $\eps \in (0, 1)$, failure probability $p \in (0, 1)$, a parameter $L \in \mathbb{R}$, and a target set $\T \subseteq \S$ such that $ECD_\GG \leq L$. Let $\pi^\star$ be a strategy profile outputted by the pair of algorithms $(\mathfrak{A}_\maxp, \mathfrak{A}_\minp)$ such that, with probability at least $1 - p$, both strategies are $\frac{\eps}2$-optimal for the game~$\GG$ with finite-horizon reachability objective to the set $\T$ of length $\frac{2(L + 1)}{\eps}$, i.e., the following two inequalities
    \begin{align}
    \Val_{R_{\frac{2(L + 1)}{\eps}}(\T)}(\mu) - \inf_{\pi_\minp \in \Pi_\minp} \PP_\mu^{(\pi^\star_\maxp, \pi_\minp)}(\omega \in \Reach_{\frac{2(L + 1)}{\eps}}(\T))  \leq \frac{\eps}2
    \label{eq:fin-eps-optimal-max}
    \end{align}
    \begin{align}
     \sup_{\pi_\maxp \in \Pi_\maxp} \PP_\mu^{(\pi_\maxp, \pi^\star_\minp)}(\omega \in \Reach_{\frac{2(L + 1)}{\eps}}(\T)) - \Val_{R_{\frac{2(L + 1)}{\eps}}(\T)}(\mu) \leq \frac{\eps}2\,
    \label{eq:fin-eps-optimal-min}
    \end{align}
    hold with probability at least $1 - p$. From this point on, we condition on this event.
    
    First, we prove that 
    \begin{align*}
        %\left | \PP_\mu^\pi \left ( \omega \in \Reach_{\frac{3(L+1)}{\eps}}(\T) \right) - \PP_\mu^\pi \left ( \omega \in \Reach(\T) \right ) \right | \le \frac{2\eps}{3}\,,
        \left|\Val_{R(\T)}(\mu) - \Val_{R_{\frac{2(L + 1)}{\eps}}(\T)}(\mu)\right| \le \frac{\eps}{2} \,. %\label{eq:claim-reduction}
    \end{align*}
    Indeed, since $\Reach_{\frac{2(L + 1)}{\eps}}(\T) \subseteq \Reach(\T)$, we have
    \begin{align}
        \Val_{R_{\frac{2(L + 1)}{\eps}}(\T)} (\mu)\leq \Val_{R(\T)}(\mu) \,. \label{eq:claim-reduction}
    \end{align}
    Therefore, we only need to show that
    \begin{align}
        \Val_{R_{\frac{2(L + 1)}{\eps}}(\T)}(\mu)
        \ge \Val_{R(\T)}(\mu) - \frac{\eps}{2}\,.
    \label{eq:fin-reduction}
    \end{align}
    From \Cref{def:ecd}, we have that for all $\pi_\minp \in \Pi_\minp$, there exists $\pi_\maxp \in \BR(\pi_\minp)$ such that
    $ETR_\GG(\pi) \leq ECD_\GG + 1$. We define a function $\varphi: \Pi_\minp \rightarrow \Pi_\maxp$ such that $ETR_\GG(\varphi(\pi_\minp), \pi_\minp) \leq ECD_\GG + 1$ and $\varphi(\pi_\minp) \in \BR(\pi_\minp)$ for all $\pi_\minp$.
    For all $\pi_\minp \in \Pi_\minp$, we have
    \begin{align}
            \sup_{\pi_\maxp \in \Pi_\maxp} \PP_\mu^{\pi} \left (\omega \in \Reach_{\frac{2(L + 1)}{\eps}}(\T) \right ) \notag &\overset{(1)} \geq \sup_{\pi_\maxp \in \BR(\pi_\minp)} \PP_\mu^{\pi} \left (\omega \in \Reach_{\frac{2(L + 1)}{\eps}}(\T) \right ) \notag \\
            & \overset{(2)}\geq \PP_\mu^{\varphi(\pi_\minp), \pi_\minp} \left (\omega \in \Reach_{\frac{2(L + 1)}{\eps}}(\T) \right ) \notag \\
            & \overset{(3)}\geq \PP_\mu^{\varphi(\pi_\minp), \pi_\minp} \left (\omega \in \Reach(\T) \right ) - \frac{\eps}{2} \notag \\
            & \overset{(4)}= 
            \sup_{\pi_\maxp \in \Pi_\maxp} \PP_\mu^{\pi} \left (\omega \in \Reach(\T) \right )  - \frac{\eps}{2}
            \label{eq:sup-of-best-response}%\\
            %& \overset{(5)}= 
            %\sup_{\pi_\maxp \in \Pi_\maxp} \PP_\mu^{\pi} \left (\omega \in \Reach(\T) \right )  - \frac{\eps}{2}\,,\label{eq:sup-of-max-all}
        \end{align}
    where $(1)$ follows from $\BR(\pi_\minp) \subseteq \Pi_\maxp$; $(2)$ follows from the definition of $\sup$ and since $\varphi(\pi_\minp) \in \BR(\pi_\minp)$; $(3)$ follows from $ETR_\GG(\varphi(\pi_\minp), \pi_\minp) \leq ECD_\GG + 1 \leq L + 1$ and by \Cref{prop:inifinte-to-finite-approx}; $(4)$ follows from the definition of best response and since $\varphi(\pi_\minp) \in \BR(\pi_\minp).$
    
    Now, we proceed to prove Inequality~\ref{eq:fin-reduction}. We have
        \begin{align*}
            \Val_{R_{\frac{2(L + 1)}{\eps}}(\T)}(\mu) & \overset{(1)}= 
            \inf_{\pi_\minp \in \Pi_\minp}\sup_{\pi_\maxp \in \Pi_\maxp} \PP_\mu^{\pi} \left (\omega \in \Reach_{\frac{2(L + 1)}{\eps}}(\T) \right ) \\
            %& \overset{(2)}\geq 
            %\inf_{\pi_\minp \in \Pi_\minp}\sup_{\pi_\maxp \in \BR(\pi_\minp)} \PP_\mu^{\pi} \left (\omega \in \Reach_{\frac{2L}{\eps}}(\T) \right ) \\
            %& \overset{(3)}\geq 
            %\inf_{\pi_\minp \in \Pi_\minp} \sup_{\pi_\maxp \in \BR(\pi_\minp)} \PP_\mu^{\pi} \left (\omega \in \Reach(\T) \right ) - \frac{\eps}{2} \\
            %& \overset{(4)}\geq 
            %\inf_{\pi_\minp \in \Pi_\minp} \sup_{\pi_\maxp \in \Pi_\maxp} \PP_\mu^{\pi_\maxp, \pi_\minp} \left (\omega \in \Reach(\T) \right ) - \frac{\eps}{2} \\
            & \overset{(2)}\geq 
            \inf_{\pi_\minp \in \Pi_\minp}
            \sup_{\pi_\maxp \in \Pi_\maxp} \PP_\mu^{\pi} \left (\omega \in \Reach(\T) \right ) - \frac{\eps}{2}
            \overset{(3)}=\Val_{R(\T)}(\mu) - \frac{\eps}{2}\,, 
        \end{align*}
    where $(1)$ follows from the definition of finite-horizon reachability value; $(2)$ follows from Inequality~\ref{eq:sup-of-best-response}; $(3)$ follows from the definition of reachability value.

    We now show that $\pi^\star$ is $\eps$-optimal for reachability objectives. For the strategy $\pi^\star_\maxp$, we have
    \begin{align*}
    &\Val_{R(\T)}(\mu) - \inf_{\pi_\minp \in \Pi_\minp} \PP_\mu^{(\pi^\star_\maxp, \pi_\minp)}(\omega \in \Reach(\T))\\
    &\overset{(1)}\leq \Val_{R(\T)}(\mu) - \inf_{\pi_\minp \in \Pi_\minp} \PP_\mu^{(\pi^\star_\maxp, \pi_\minp)}(\omega \in \Reach_{\frac{2(L + 1)}{\eps}}(\T))\\
    &\overset{(2)}\leq \frac{\eps}{2} + \Val_{R_{\frac{2(L + 1)}{\eps}}(\T)}(\mu) - \inf_{\pi_\minp \in \Pi_\minp} \PP_\mu^{(\pi^\star_\maxp, \pi_\minp)}(\omega \in \Reach_{\frac{2(L + 1)}{\eps}}(\T)
    \overset{(3)}\leq \eps\,,
    \end{align*}
    where $(1)$ follows from $\Reach_{\frac{2(L + 1)}{\eps}}(\T) \subseteq \Reach(\T)$; $(2)$ follows from Inequality~\ref{eq:fin-reduction}; and $(3)$ follows from Inequality~\ref{eq:fin-eps-optimal-max}. 

     For the strategy $\pi^\star_\minp$, we have
    \begin{align*}
     &\sup_{\pi_\maxp \in \Pi_\maxp} \PP_\mu^{(\pi_\maxp, \pi^\star_\minp)}(\omega \in \Reach(\T)) - \Val_{R(\T)}(\mu)\\
     &\overset{(1)}\leq \sup_{\pi_\maxp \in \Pi_\maxp} \PP_\mu^{(\pi_\maxp, \pi^\star_\minp)}(\omega \in \Reach(\T)) -  \Val_{R_{\frac{2(L + 1)}{\eps}}(\T)}(\mu)\\
     %&\overset{(2)} = \sup_{\pi_\maxp \in \BR(\pi_\minp^\star)} \PP_\mu^{(\pi_\maxp, \pi^\star_\minp)}(\omega \in \Reach(\T)) -  \Val_{R_{\frac{2L}{\eps}}(\T)}(\mu)\\
     %&\overset{(3)} \leq \sup_{\pi_\maxp \in \BR(\pi_\minp^\star)} \PP_\mu^{(\pi_\maxp, \pi^\star_\minp)}(\omega \in \Reach_{\frac{2L}{\eps}}(\T)) + \frac{\eps}{2} -  \Val_{R_{\frac{2L}{\eps}}(\T)}(\mu)\\
     &\overset{(2)} \leq \sup_{\pi_\maxp \in \Pi_\maxp} \PP_\mu^{(\pi_\maxp, \pi^\star_\minp)}(\omega \in \Reach_{\frac{2(L + 1)}{\eps}}(\T)) + \frac{\eps}{2} -  \Val_{R_{\frac{2(L + 1)}{\eps}}(\T)}(\mu)
     \overset{(3)} \leq \eps\,,
    \end{align*}
    where $(1)$ follows from Inequality~\ref{eq:claim-reduction}; $(2)$ follows from Inequality~\ref{eq:sup-of-best-response}, $(3)$ follows from Inequality~\ref{eq:fin-eps-optimal-min}. 
\end{proof}

\section{Proofs of \texorpdfstring{\Cref{sec:pac-rl-for-finite}}{Section~\ref{sec:pac-rl-for-finite}}}
\backwardinduction*
\label{app:proofs:algorithm}

\begin{proof}
We prove both items by an induction on the step $\ell$.

\noindent{\emph{Induction Base $(\ell = L)$.}} 
We have
\begin{align*}
    V_{B, W}^{\pi}(s, L) \overset{(1)}= \mathbb{1}((s, L) \in W) \overset{(2)}= \PP_{(s, \ell)}^{\pi} \left( \left (s_{L} \right ) \in \safereach \left (B, W \right ) \right )\,,
\end{align*}
where $(1)$ and $(2)$ are by the definitions. 

% Similarly for Item $(b)$, we have
% \begin{align*}
%     Q_{\T', \unexplored_{\minp}^{q-1}}^{\pi^q}&(s, \ell, a)\\
%     &\overset{(1)}= \sum_{s' \in \S}\Prob(s, a)(s')\mathbb{1}((s', L) \in \unexplored_\minp^{q-1})\\
%     &\overset{(2)}= \PP_{\Prob'((s, \ell), a)}^{\pi'} \left( \left (s_{L} \right ) \in \safereach \left (\T', \unexplored_{\minp}^{q-1} \right ) \right )\,,
% \end{align*}
% where $(1)$ and $(2)$ are by the definitions.

% One has that $Result = \sum_{t \in \S}\Prob(s, a)(t)\mathbb{1}((t, L) \in \unexplored_\minp^{q-1})$ which is equal to $Q_{\T', \unexplored_{\minp}^{q-1}}^{\pi^q}(s, \ell, a)$ by the definition.

\noindent{\emph{Induction Step $(\ell < L)$.}} We first show the second item of the result and then we prove the first item. 
We assume the claims hold for $\ell+1$. We have
\begin{align*}
    Q_{B, W}^{\pi}(s, \ell, a)
    \overset{(1)}= \sum_{s' \in \S} \Prob(s, a)(s') V_{B, W}^{\pi}(s', \ell+1)\,,%\label{eq:Qa_def}
\end{align*}
where $(1)$ follows from the definition of $Q_{B, W}^{\pi}(s, \ell, a)$.

By the induction hypothesis, we have
$V^{\pi}_{B, W}(s',\ell+1) = \mathbb{P}^{\pi}_{(s',\ell+1)}
\Bigl(
(s_{\ell+1},a_{\ell+1},\ldots,s_L)\in \safereach(B, W)
\Bigr)$.
Therefore, we have $Q^{\pi}_{B, W}(s,\ell,a) =
\sum_{s'\in S}\delta(s,a)(s')\cdot
\mathbb{P}^{\pi}_{(s',\ell+1)}
\Bigl(
(s_{\ell+1},a_{\ell+1},\ldots,s_L)\in \safereach(B, W)
\Bigr)$.
Finally, by the law of total probability, the right-hand side is exactly the probability of the same event when the initial state-step is drawn from $\delta'((s,\ell),a)$ (i.e., $(s',\ell+1)$ is chosen with probability $\delta(s,a)(s')$):
\[
\sum_{s'\in S}\delta(s,a)(s')\cdot
\mathbb{P}^{\pi}_{(s',\ell+1)}(\cdot)
=
\mathbb{P}^{\pi}_{\delta'((s,\ell),a)}(\cdot),
\]
which proves the second item.

We now prove the first item. 
There are three cases.

\smallskip
\noindent\emph{Case 1: $(s,\ell)\in B$.} Then $V^{\pi}_{B, W}(s,\ell)=0$ by \Cref{def:qvalue}. Also, by starting the suffix from $(s,\ell)$, the condition of the event $\safereach(B, W)$ is violated immediately. Hence
\[
\mathbb{P}^{\pi}_{(s,\ell)}
\Bigl(
(s_{\ell},a_{\ell},\ldots,s_L)\in \safereach(B, W)
\Bigr)=0.
\]

\smallskip
\noindent\emph{Case 2: $(s,\ell)\in  W$.} Then $V^{\pi}_{B, W}(s,\ell)=1$ by \Cref{def:qvalue}. By starting the suffix from $(s,\ell)$, the event $\safereach(B, W)$ holds immediately. Thus
\[
\mathbb{P}^{\pi}_{(s,\ell)}
\Bigl(
(s_{\ell},a_{\ell},\ldots,s_L)\in \safereach(B, W)
\Bigr)=1.
\]

\smallskip
\noindent\emph{Case 3: $(s,\ell)\notin B \cup W$.}
Then \Cref{def:qvalue} gives
\[
V^{\pi}_{B, W}(s,\ell)=Q^{\pi}_{B, W}\bigl(s,\ell,\pi(s,\ell)\bigr).
\]
By the induction hypothesis for the second item applied at step $\ell$ with state $s$ and action $\pi(s,\ell)$, we have
\begin{align*}
    Q^{\pi}_{B, W}\bigl(s,\ell,\pi(s,\ell)\bigr)\notag = \mathbb{P}^{\pi}_{\delta'((s,\ell),\,\pi(s,\ell))}
    \Bigl(
    (s_{\ell+1},a_{\ell+1},\ldots,s_L)\in \safereach(B, W)
    \Bigr).%\label{eq:piq_suffix}
\end{align*}
Since $(s,\ell)\notin B \cup W$, the right-hand side is exactly the probability of satisfying $\safereach(B, W)$ starting from $(s,\ell)$ under $\pi$, after taking the first (deterministic) action $\pi(s,\ell)$; hence this equals
\[
\mathbb{P}^{\pi}_{(s,\ell)}
\Bigl(
(s_{\ell},a_{\ell},\ldots,s_L)\in \safereach(B, W)
\Bigr)
\]
which closes the third case.

\smallskip
\noindent
Combining the three cases proves the first item and completes the proof.
\end{proof}

\event*
\begin{proof}
We prove the probability of sub-events $(a)$ and $(c)$ is at least $1 - p/2$ and the sub-event $(b)$ follows directly from the sub-event $(a)$. Taking a union bound gives us the desired probability. 

Let $q \in [|\S'|]$ be a stage and $(s, \ell) \in \S'$ be a state. 
\begin{enumerate}[(a)]
    \item \label{item:reachability} 
    
    %We assume \cref{eq:probability} and $s \in \S_1$ since $\unexplored^{q+1} \subseteq \S_1 \times [L-1]$. 
    For a stage $r \in [q-1]$ and $k \in [K]$, let $X_{r, k}$ be the event of reaching a state-step $(s, \ell)$ in a play induced by the strategy profile $\pi^r$. We define $X \defas  \sum_{r \in [q-1], k \in [K]}X_{r, k}.$ By \Cref{lem:connection}, we have $\EE(X) = K \sum_{r \in [q-1]} V_{\emptyset, \{(s, \ell)\}}^{\pi^r}$.

    %If $(s, \ell) \not\in \unexplored^q$ then we are done since $\unexplored^{q+1} \subseteq \unexplored^q$. Hence, we assume $(s, \ell) \in \unexplored^q$. Then, by the definition, we have that $(s, \ell) \not\in \unexplored^{q+1}$ iff $\counted^{q, \ell}(s) \geq pK$.
    
    Lines \ref{line:strategy} to \ref{line:count} show that $\counted^{q}(s, \ell) = X$ since the value of $\counted^{q}(s, \ell)$ only depends on the first $\ell-1$ steps of the sampled plays which are driven by the strategy profiles $\pi^0, \ldots, \pi^{q-1}$.
    
    Thus, we obtain
    \begin{align*}
        &\PP \left ( \left | \counted^{q}(s, \ell) - K\sum_{r \in [q-1]} V_{\emptyset, \{(s, \ell)\}}^{\pi^r} \right | \leq K\eest \right ) \overset{(1)}=\PP \left (|X - \E(X)| \leq K\eest \right) \\
         &\overset{(2)}\geq 1 - 2exp \left (-(K\eest)^2 \right )
         \overset{(3)}\geq 1 - 2exp(-K\eest)
         \overset{(4)}\geq 1 - \frac{p}{2|\S|^2L^2},
    \end{align*}
    where $(1)$ follows from $X = \counted^{q}(s, \ell)$ and $\EE(X) = K \sum_{r \in [q-1]} V_{\emptyset, \{(s, \ell)\}}^{\pi^r}$, $(2)$ follows from Hoeffding's inequality, $(3)$ follows from $K\eest \ge 1$, and $(4)$ follows from $K\eest \ge \log \left (\frac{|\S|^2 L^2}{4p} \right )$.
    
    Taking the union bound for all stages, steps and states, we get that the probability of the sub-event is $1 - p/2$.

    \item This is a direct consequence of the sub-event $(a)$. Assuming $3K\eest \leq \counted^{q}(s, \ell)$ we get that
    \begin{align*}
        3K\eest &\leq \counted^{q}(s, \ell) &\\
        &\leq K\sum_{r \in [q-1]} V_{\{(s, \ell)\}}^{\pi^r} + K\eest& \text{(By \cref{eq:probability})}\\
        &\leq K\sum_{r \in [q'-1]} V_{\{(s, \ell)\}}^{\pi^r} + K\eest&(q' \geq q)\\
        &\leq \counted^{q'}(s, \ell) + 2K\eest&\text{(By \cref{eq:probability})}
    \end{align*}

    \item 
    %If $(s, l) \in \unexplored^{q+1}$ then the claim follows trivially. Let us now assume $(s, l) \not\in \unexplored^{q+1}$. which implies $s \not\in \wellexp^{q, l}$
    We prove only the first claim as the second one is proven analogously.

    The statement holds trivially for $s$ such that $(s, \ell) \in \unexplored_\maxp^{q} \cup \T_\maxp$. 
    Hence, we assume $(s, \ell) \not\in \unexplored^{q}_\maxp$ which means there exists $q' < q$ such that $\counted^{q'}(s, \ell) \geq 3K\eest$ (by the condition at Line~\ref{line:condition1}). Assuming sub-events $(a)$ and $(b)$, this implies $\counted^{q}(s, \ell) \geq K\eest$. Hence, the best arm identification routine (Line~\ref{line:bai}) was called at least
    $K\eest \geq \frac{C|\A|\log(|\S|^2L^2/p)}{\eerr^2}$
    number of times which is sufficient to obtain $(\eerr, \frac{p}{|\S^2|L^2})$-PAC guarantees on the selected arm \cite{DBLP:journals/jmlr/Even-DarMM06}[Theorem 10].
    First, we need to show that for each action $a \in \A$, the samples were independent random variables from the same Bernoulli distribution. This is clear by the definition of the strategy in Line~\ref{line:strategy} since the part of the play from the step $\ell+1$ onwards is driven by the strategy profile $\pi^q$ which does not change during the sampling. Therefore, sampling an action $a$ in a state-step $(s, \ell)$ corresponds to sampling a Bernoulli random variable with the value defined by Line~\ref{line:play} to Line~\ref{line:result}, i.e., the random play that starts in $(s, \ell)$ satisfies the condition defined in Line~\ref{line:winning}. By \Cref{lem:connection}, this is equal to $Q_{\emptyset, \unexplored_\maxp^{q} \cup \T_\maxp}^{\pi^q}(s, \ell, a)$.

    Since the output of the best-arm identification routine is the action $\pi_\maxp^q(s, \ell)$, we obtain the desired inequality
    for all $a \in A$ with probability $1 - \frac{p}{|S|^2L^2}$.

    Taking a union over all stages, steps, and states, we obtain that the probability of the sub-event is $1 - p/2$.
\end{enumerate}
\end{proof}

\expandedgamevalue*

\begin{proof}
We start by proving the first claim.

Since there always exists an optimal positional strategy for the expanded game (see \Cref{rem:positional}), we assume $\pi_\maxp$ and $\pi_\minp$ to be positional.

We prove the claim by an induction on $\ell$, i.e., we prove that for all states $s \in \S$, we have
\begin{align}\label{eq:induction}
V_{\emptyset, \unexplored_\maxp^f \cup \T_\maxp}^{\pi_\maxp, \pi_\minp^{f}}( s, \ell) - V_{\emptyset, \unexplored_\maxp^f \cup \T_\maxp}^{\pi^{f}}( s, \ell) \leq (L - l)\eerr.
\end{align}

\noindent{\emph{Induction Base $(\ell = L)$.}}
The base case $\ell = L$ is trivial as both values are either 0 or 1 depending on whether $s \in \unexplored^f \cup \T_\maxp$ and the strategies play no role.

\noindent{\emph{Induction Step $(\ell < L)$.}} We assume \Cref{eq:induction} for $\ell+1$. There are two cases. 

\emph{Case $(s, \ell)  \not\in \unexplored_\maxp^f \cup \T_\maxp$:} Recall that $\Prob(s, a)(s')$ denotes the probability of reaching a state $s'$ from a state $s$ by playing an action~$a$. First, we assume that $s \in \S_\maxp$. Therefore, we get
\begin{align*}\label{eq:inductive}
&V_{\emptyset,\unexplored_\maxp^f \cup \T_\maxp}^{\pi_\maxp, \pi_\minp^{f}}(s, \ell) - V_{\emptyset,\unexplored_\maxp^f \cup \T_\maxp}^{\pi^{f}}(s, \ell)\\
&\overset{(1)}= Q_{\emptyset,\unexplored_\maxp^f \cup \T_\maxp}^{\pi_\maxp, \pi_\minp^{f}}(s, \ell, \pi_\maxp(s, \ell)) - Q_{\emptyset,\unexplored_\maxp^f \cup \T_\maxp}^{\pi^{f}}(s, \ell, \pi_\maxp^{f}(s, \ell)) \\ 
&\overset{(2)}= \sum_{s' \in \S}\Prob(s, \pi_\maxp(s, \ell))(s')V_{\emptyset,\unexplored_\maxp^f \cup \T_\maxp}^{\pi_\maxp, \pi_\minp^{f}}(s', \ell+1)-Q_{\emptyset,\unexplored_\maxp^f \cup \T_\maxp}^{\pi^{f}}(s, \ell, \pi_\maxp^{f}(s, \ell))\\
&\overset{(3)}\leq \sum_{s' \in \S}\Prob(s, \pi_\maxp(s, \ell))(s')V_{\emptyset,\unexplored_\maxp^f \cup \T_\maxp}^{\pi^{f}}(s', \ell+1)-Q_{\emptyset,\unexplored_\maxp^f \cup \T_\maxp}^{\pi^{f}}(s, \ell, \pi_\maxp^{f}(s, \ell))\\
&\quad \quad+ (L - \ell-1)\eerr\\
&\overset{(4)}= Q_{\unexplored_\maxp^f \cup \T_\maxp}^{\pi^{f}}(s, \ell, \pi_\maxp(s, \ell))-Q_{\unexplored_\maxp^f \cup \T_\maxp}^{\pi^{f}}(s, \ell, \pi_\maxp^{f}(s, \ell)) + (L - \ell-1)\eerr\\
&\overset{(5)}\leq (L - \ell)\eerr
\end{align*}
where $(1), (2)$ and $(4)$ follow from \Cref{def:qvalue}, $(3)$ follows from the inductive assumption and $(5)$ follows from the event $E$ part $(c)$.
If $s \in S_\minp$, applying the same steps gives us even tighter bound $(L - l -1)\eerr$.

\emph{Case $(s, \ell) \in \unexplored^f \cup \T_\maxp$:} This case is trivial due to \Cref{def:qvalue} and the fact that strategies play no role.

\noindent Altogether, by \Cref{def:qvalue}, we have
\[
    V_{\emptyset,\unexplored_\maxp^f \cup \T_\maxp}^{\pi_\maxp, \pi_\minp^{f}} = \sum_{s \in \S}\mu(s)V_{\emptyset, \unexplored_\maxp^f \cup \T_\maxp}^{\pi_\maxp, \pi_\minp^{f}}(s, 1)\,,
\]
which proves the first claim.

We now prove the second claim similarly by induction on $\ell$:
\begin{align}
    V_{\T_\maxp, \unexplored_\minp^f \cup \T_\minp}^{\pi_\maxp^f, \pi_\minp}( s, \ell) - V_{\T_\maxp, \unexplored_\minp^f \cup \T_\minp}^{\pi^{f}}( s, \ell) \leq (L - l)\eerr.
\end{align}
\noindent{\emph{Induction based $(\ell = L)$}.} The base case $\ell = L$ is trivial. 

\noindent{\emph{Induction Step $(\ell < L)$}.} There are two cases. 

\emph{Case $(s, \ell) \not\in \unexplored_\minp^f \cup \T_\maxp \cup \T_\minp$:} We first assume $s \in \S_\minp$. Therefore, we get
\begin{align*}\label{eq:inductive}
&V_{\T_\maxp, \unexplored_\minp^f \cup \T_\minp}^{\pi_\maxp^f, \pi_\minp}(s, \ell) - V_{\T_\maxp, \unexplored_\minp^f \cup \T_\minp}^{\pi^{f}}(s, \ell)\\
&\overset{(1)}= Q_{\T_\maxp, \unexplored_\minp^f \cup \T_\minp}^{\pi_\maxp^f, \pi_\minp}(s, \ell, \pi_\minp(s, \ell)) - Q_{\T_\maxp, \unexplored_\minp^f \cup \T_\minp}^{\pi^{f}}(s, \ell, \pi_\minp^{f}(s, \ell)) \\ 
&\overset{(2)}= \sum_{s' \in \S}\Prob(s, \pi_\minp(s, \ell))(s')V_{\T_\maxp, \unexplored_\minp^f \cup \T_\minp}^{\pi_\minp, \pi_\minp^{f}}(s', \ell+1)-Q_{\T_\maxp, \unexplored_\minp^f \cup \T_\minp}^{\pi^{f}}(s, \ell, \pi_\minp^{f}(s, \ell))\\
&\overset{(3)}\leq \sum_{s' \in \S}\Prob(s, \pi_\minp(s, \ell))(s')V_{\T_\maxp, \unexplored_\minp^f \cup \T_\minp}^{\pi^{f}}(s', \ell+1)-Q_{\T_\maxp, \unexplored_\minp^f \cup \T_\minp}^{\pi^{f}}(s, \ell, \pi_\minp^{f}(s, \ell))\\
&\quad \quad + (L - \ell-1)\eerr\\
&\overset{(4)}= Q_{\T_\maxp, \unexplored_\minp^f \cup \T_\minp}^{\pi^{f}}(s, \ell, \pi_\minp(s, \ell))-Q_{\T_\maxp, \unexplored_\minp^f \cup \T_\minp}^{\pi^{f}}(s, \ell, \pi_\minp^{f}(s, \ell))+ (L - \ell-1)\eerr\\
&\overset{(5)}\leq (L - \ell)\eerr
\end{align*}
where $(1), (2)$ and $(4)$ follow from the \Cref{def:qvalue}, $(3)$ follows from the inductive assumption and $(5)$ follows from the event $E$ part~$(c)$. If $s \in S_\maxp$, applying the same steps gives us even tighter bound $(L - l -1)\eerr$.

\emph{Case $(s, \ell) \in \unexplored^f \cup \T_\maxp \cup \T_\minp$:} This is trivial.

\noindent Altogether, by \Cref{def:qvalue}, we have
\[
V_{\T_\maxp, \unexplored_\minp^f \cup \T_\minp}^{\pi_\maxp^f, \pi_\minp} = \sum_{s \in \S}\mu(s)V_{\T_\maxp, \unexplored_\minp^f \cup \T_\minp}^{\pi_\maxp^f, \pi_\minp}(s, 1)\,,
\]
which yields the result.

\end{proof}

\expandedvalue*

\begin{proof}
Observe that for all strategy profiles $\pi$, we have
\begin{align*}
    \PP_\mu^\pi \left ( \omega \in \Reach_L(\T) \right ) &= \PP_{\mu'}^\pi \left ( (s_0,a_0,\ldots,s_L) \in \safereach(\emptyset, \T_\maxp) \right)\\ 
    &= 1 - \PP_{\mu'}^\pi \left ( (s_0,a_0,\ldots,s_L) \in \safereach(\T_\maxp, \T_\minp \right)\,.
\end{align*}
Therefore, by \Cref{lem:connection} Item~$(1)$, we only need to prove the following statements. 
\begin{itemize}
    \item For all strategies $\pi_\maxp$ of the player $\maxp$, we have
    \[
        V_{\emptyset, \T_\maxp}^{\pi_\maxp, \pi_\minp^{f}} - V_{\emptyset, \T_\maxp}^{\pi^{f}} \leq \varepsilon\,;
    \]
    \item for all strategies $\pi_\minp$ of the player $\minp$, we have
    \[
        V_{\T_\maxp, \T_\minp}^{\pi_\maxp^f, \pi_\minp} - V_{\T_\maxp, \T_\minp}^{\pi^{f}} \leq \varepsilon\,.
    \]
\end{itemize}

For the first claim, we have
\begin{align*}
V_{\emptyset, \T_\maxp}^{\pi_\maxp, \pi_\minp^{f}} - V_{\emptyset, \T_\maxp}^{\pi^{f}} &\overset{(1)}\leq V_{\emptyset, \unexplored_\maxp^{f} \cup \T_\maxp}^{\pi_\maxp, \pi_\minp^{f}} - V_{\emptyset, \T_\maxp}^{\pi^{f}}\\
&\overset{(2)}= V_{\emptyset, \unexplored_\maxp^{f} \cup \T_\maxp}^{\pi_\maxp, \pi_\minp^{f}} - V_{\emptyset, \T_\maxp}^{\pi^{f}} - V_{\emptyset, \unexplored^f_\maxp}^{\pi^{f}} + V_{\emptyset, \unexplored^f_\maxp}^{\pi^{f}}\\
&\overset{(3)}\leq V_{\emptyset, \unexplored_\maxp^{f} \cup \T_\maxp}^{\pi_\maxp, \pi_\minp^{f}} - V_{\emptyset, \T_\maxp}^{\pi^{f}} - V_{\emptyset, \unexplored_\maxp^{f}}^{\pi^{f}} + 4|\S'|\eest \\
&\overset{(4)}\leq V_{\emptyset, \unexplored_\maxp^{f} \cup \T_\maxp}^{\pi_\maxp, \pi_\minp^{f}} - V_{\emptyset, \unexplored_\maxp^{f} \cup \T_\maxp}^{\pi^{f}} + 4|\S'|\eest\\
&\overset{(5)}\leq L\eerr + 4|\S'|\eest
\overset{(6)}\leq \eps,
\end{align*}
where $(1)$ follows from the fact that $\T_\maxp \subseteq U^f_\maxp \cup \T_\maxp$, $(2)$ follows from algebraic manipulation, $(3)$ follows from \Cref{lem:unexplored}, $(4)$ follows from $V_{\emptyset, \unexplored_\maxp^{f} \cup \T_\maxp}^{\pi^{f}} \le V_{\emptyset, \T_\maxp}^{\pi^{f}} + V_{\emptyset, \unexplored_\maxp^{f}}^{\pi^{f}}$, $(5)$ follows from \Cref{lem:constructed}, and $(6)$ follows from \Cref{def:alg-constants}. 

Similarly for the second claim, we have
\begin{align*}
V_{\T_\maxp, \T_\minp}^{\pi_\maxp^f, \pi_\minp} - V_{\T_\maxp, \T_\minp}^{\pi^{f}} &\overset{(1)}\leq V_{\T_\maxp, U_\minp^{f} \cup \T_\minp}^{\pi_\maxp^f, \pi_\minp} - V_{\T_\maxp, \T_\minp}^{\pi^{f}} \\
&\overset{(2)}\leq V_{\T_\maxp, U_\minp^{f} \cup \T_\minp}^{\pi_\maxp^f, \pi_\minp} - V_{\T_\maxp, \T_\minp}^{\pi^{f}} - V_{\T_\maxp, \unexplored^f_\minp}^{\pi^{f}} + V_{\T_\maxp, \unexplored^f_\minp}^{\pi^{f}}\\
&\overset{(3)}\leq V_{\T_\maxp, U_\minp^{f} \cup \T_\minp}^{\pi_\maxp^f, \pi_\minp} - V_{\T_\maxp, \T_\minp}^{\pi^{f}} - V_{\T_\maxp, \unexplored^f_\minp}^{\pi^{f}}  + 4|\S'|\eest \\
&\overset{(4)}\leq V_{\T_\maxp, U_\minp^{f} \cup \T_\minp}^{\pi_\maxp^f, \pi_\minp} - V_{\T_\maxp, U_\minp^f \cup \T_\minp}^{\pi^{f}} + 4|\S'|\eest \\
&\overset{(5)}\leq L\eerr + 4|\S'|\eest
\overset{(6)}\leq \eps\,,
\end{align*}
where $(1)$ follows from the fact that $\T_\minp \subseteq U^f_\minp \cup \T_\minp$, $(2)$ follows from algebraic manipulation, $(3)$ follows from \Cref{lem:unexplored}, $(4)$ follows from $V_{\T_\maxp, \unexplored_\minp^{f} \cup \T_\minp}^{\pi^{f}} \le V_{\T_\maxp, \T_\minp}^{\pi^{f}} + V_{\T_\maxp, \unexplored_\minp^{f}}^{\pi^{f}}$, $(5)$ follows from \Cref{lem:constructed}, and $(6)$ follows from \Cref{def:alg-constants}. 

%It is helpful to notice that $V_{\emptyset, U_\minp^f}^{\pi^f} \leq 4|\S'|\eest$ by \cref{lem:unexplored}; and $V_{A, B}^{\pi} + V_{\emptyset, C}^{\pi} \geq V_{A, B \cup C}^{\pi}$ where $A, B, C \subseteq \S'$ are pairwise disjoint.
\end{proof}

\section{Outlines of Algorithms}
\label{app:sec:algorithms}

\begin{algorithm}
\SetAlgoLined
\LinesNumbered
\SetKwFunction{FMain}{Best-Arm-Identification}
\SetKwProg{Fn}{Function}{:}{}
\KwData{$\A, \eps, p$}
%\Fn{\FMain{$\mathcal{A}, \delta, \varepsilon$}}{
$\A' = \A, \eps' = \eps/4, p' = p/2$\;
\While{$|\A'| > 1$} {
    For all $a \in \A'$, sample $a$ for $\frac{\log(3/p')}{(\eps'/2)^2}$ number of times and let $\bar a$ be its empirical value\;
    Let $m$ be the median value of $ \{\bar a ~|~ a \in \A'$\}\;
    $\A' \gets \A' \setminus \{a \in \A' ~|~ \bar a < m\}$\;
    $\eps' \gets \eps'\frac 34; p' \gets p'\frac 12$\;
}
\KwRet{$a \in \A'$}
%}

\caption{Best-Arm Identification Algorithm from \cite{DBLP:journals/jmlr/Even-DarMM06}[Section 3.2]}
\label{alg:bai}
\end{algorithm}

\end{document}